\documentclass[twoside,11pt]{article}

\usepackage[utf8]{inputenc} 
\usepackage[T1]{fontenc} 
\usepackage{times}
\usepackage{color}
\usepackage{multirow}
\usepackage{adjustbox} 

\usepackage{multirow}
\usepackage{subcaption}

\usepackage{times}
\usepackage{latexsym}
\usepackage{courier}
\usepackage{graphicx}
\usepackage{amsmath}
\usepackage{amssymb}
\usepackage{multirow}
\usepackage{dsfont}
\usepackage{booktabs}
\usepackage{array}
\usepackage{graphicx}
\usepackage{todonotes}
\usepackage{picinpar}
\usepackage{multirow}
\usepackage[normalem]{ulem}

\usepackage{arydshln}
\allowdisplaybreaks

\usepackage{amsthm}
\usepackage{pdfpages} 
\usepackage[ruled, vlined, linesnumbered, noresetcount]{algorithm2e}

\usepackage{caption}
\usepackage{subcaption}
\usepackage{booktabs}

\newcommand{\squishlist}{
\begin{list}{$\bullet$}{ 
    \setlength{\itemsep}{1pt}
    \setlength{\parsep}{0pt}
    \setlength{\topsep}{1.5pt}
    \setlength{\partopsep}{0pt}
    \setlength{\leftmargin}{1em}
    \setlength{\labelwidth}{1em}
    \setlength{\labelsep}{0.5em} } }
\newcommand{\squishend}{\end{list}  }

\newcommand{\squishenum}{
\begin{list}{$\bullet$}{ 
    \setlength{\itemsep}{1pt}
    \setlength{\parsep}{0pt}
    \setlength{\topsep}{1.5pt}
    \setlength{\partopsep}{0pt}
    \setlength{\leftmargin}{2em}
    \setlength{\labelwidth}{1.5em}
    \setlength{\labelsep}{0.5em} } }
    
\newcommand{\setf}[1]{{\bf{#1}}}
\newcommand{\scope}[1]{\mathbf{x}^{#1}}

\newcommand{\argmin}{\operatornamewithlimits{argmin}}

\newtheorem{theorem}{{\bf Theorem}}

\newtheorem{definition}{{\bf Definition}}

\newtheorem{example}{{\bf Example}}

\newboolean{includeMemo}
\setboolean{includeMemo}{true} 

\newcommand{\memoside}[1]{\ifthenelse{\boolean{includeMemo}}{\todo[caption={},color=green!20!]{{\footnotesize #1}}}}
\newcommand{\memo}[1]{\ifthenelse{\boolean{includeMemo}}{\todo[inline,caption={},color=green!20!]{#1}}}
\newcommand{\highlightbox}[1]{\ifthenelse{\boolean{includeMemo}}{\todo[inline,caption={},color=gray!20!]{#1}}}

\newcommand{\xhdr}[1]{\vspace{5pt}\noindent\textbf{#1 }}
\newcommand{\ignore}[1]{}

\usepackage{ulem}
\newcommand{\var}[1]{{\textit{var}(#1)}}

\usepackage{url}

\usepackage{jair, theapa, rawfonts}

\ShortHeadings{Explainable Distributed Constraint Optimization Problems}
{Rachmut, Vasileiou, Weinstein, Zivan \& Yeoh}
\firstpageno{1}

\begin{document}

\title{Explainable Distributed Constraint Optimization Problems}

\author{\name Ben Rachmut \email benr@wustl.edu\\
\addr Department of Computer Science and Engineering\\
Washington University in St. Louis
\AND
\name Stylianos Loukas Vasileiou \email v.stylianos@wustl.edu \\
\addr Department of Computer Science and Engineering\\
Washington University in St. Louis
\AND
\name Nimrod Meir Weinstein \email nimrodme@post.bgu.ac.il \\
\addr  Department of Industrial Engineering and Management\\
Ben-Gurion University of the Negev 
\AND
\name Roie Zivan \email zivanr@bgu.ac.il \\
\addr  Department of Industrial Engineering and Management\\
Ben-Gurion University of the Negev 
\AND
\name William Yeoh \email wyeoh@wustl.edu \\
\addr Department of Computer Science and Engineering\\
Washington University in St. Louis
      }


\maketitle

\begin{abstract}
The Distributed Constraint Optimization Problem (DCOP) formulation is a powerful tool to model cooperative multi-agent problems that need to be solved distributively. A core assumption of existing approaches is that DCOP solutions can be easily understood, accepted, and adopted, which may not hold, as evidenced by the large body of literature on Explainable AI. In this paper, we propose the Explainable DCOP (X-DCOP) model, which extends a DCOP to include its solution and a contrastive query for that solution. We formally define some key properties that contrastive explanations must satisfy for them to be considered as valid solutions to X-DCOPs as well as theoretical results on the existence of such valid explanations. To solve X-DCOPs, we propose a distributed framework as well as several optimizations and suboptimal variants to find valid explanations. We also include a human user study that showed that users, not surprisingly, prefer shorter explanations over longer ones. Our empirical evaluations showed that our approach can scale to large problems, and the different variants provide different options for trading off explanation lengths for smaller runtimes. Thus, our model and algorithmic contributions extend the state of the art by reducing the barrier for users to understand DCOP solutions, facilitating their adoption in more real-world applications. 

\end{abstract}

\section{Introduction}


The \emph{Distributed Constraint Optimization Problem} (DCOP)~\cite{modi:05,petcu:05,fioretto:18} formulation is a powerful tool to model cooperative multi-agent optimization problems. DCOPs are used for modeling many problems that are distributed by nature and where agents need to coordinate their decisions (represented by value assignments) to minimize the aggregated constraint costs. This model is widely employed 
for representing distributed problems such as meeting scheduling~\cite{maheswaran:04a}, wireless sensor networks~\cite{yeoh:10}, multi-robot teams coordination~\cite{zivan:15,pertzovskiy2023cams}, smart grids~\cite{miller:12}, smart homes~\cite{fioretto:17a,rust2022resilient}, and large-scale satellite constellations~\cite{ZilbersteinRSC24}.

One of the core assumptions of the existing approaches is that DCOP solutions can be easily understood, accepted, and adopted in these various applications. Unfortunately, this assumption is often invalid and there is often a strong need for AI systems to explain their recommendations or solutions to human users, as evidenced by the large body of literature on Explainable AI (XAI)~\cite{gunning2019xai,Miller2019,ttathai2020,byrne2023good}. This need is even more important in applications where DCOP agents represent human users who are affected by the choice of value assignments in the DCOP solution. For example, in a meeting scheduling problem, it is important for human users to be able to query their agents to understand why their meetings are scheduled at particular times. 

To address this limitation, in this paper, we propose the \emph{Explainable DCOP} (X-DCOP) model, which extends a DCOP to include its solution, and a query regarding that solution. Motivated by insights from the broader XAI literature, we consider \emph{contrastive queries} due to their relevance in decision-making scenarios~\cite{krarup2021contrastive,vasileioua2023lasp,zehtabi2024contrastive}. Contrastive queries are of the form ``Why is solution A chosen over an alternative solution B?'' or, in the case of X-DCOPs, ``Why is variable $x_1$ assigned value $d_1$ instead of an alternative value $d_2$?''. A solution to an X-DCOP is a \emph{contrastive explanation} that provides counterfactual reasoning on the impact of assigning the alternative value instead of the original value (i.e., the increase in constraint cost and the source of this increase), justifying why the original value is better than the alternative. 

To solve X-DCOPs, we propose a distributed framework called CEDAR as well as several optimizations and suboptimal variants to compute the explanations. We also include a human user study to evaluate the user understanding and satisfaction of the explanations; the study shows that users, not surprisingly, prefer shorter explanations over longer ones. Finally, our empirical computational results on random graphs and meeting scheduling benchmarks show that CEDAR can scale to large X-DCOP problems with 50 agents, and the different versions of CEDAR trade off explanation lengths for smaller runtimes. 


\section{Background: DCOPs}

A \emph{Distributed Constraint Optimization Problem} (DCOP)~\cite{modi:05,petcu:05,fioretto:18} is a tuple  
$\langle \setf{A}, \setf{X}, \setf{D}, \setf{F}, \alpha \rangle$, where: 
\squishlist
\item $\setf{A} = \{a_i\}_{i=1}^p$ is a set of \emph{agents}. 
\item $\setf{X} = \{x_i\}_{i=1}^n$ is a set of \emph{variables}. 
\item $\setf{D} = \{D_x\}_{x \in \setf{X}}$ is a set of finite \emph{domains} and each variable $x \in \setf{X}$ can be assigned values from the set $D_{x}$. 
\item $\setf{F} = \{f_i\}_{i=1}^m$ is a set of \emph{constraints}, each defined over a set of variables: $ f_i : \prod_{x \in \scope{f_i}} D_x \to \mathbb{R} \cup \{\infty\}$, where infeasible configurations have $\infty$ costs and $\scope{f_i} \subseteq \setf{X}$ is the \emph{scope} of $f_i$. 
\item $\alpha : \setf{X} \to \setf{A}$ is a \emph{mapping function} that associates each variable to one agent. 
\squishend

A \emph{solution} $\sigma$ is a value assignment for a set of variables that is consistent with their respective domains. We use the term $\var{\sigma}$ to denote the set of variables whose values are assigned in $\sigma$. 
A \emph{grounded constraint} $f_{\downarrow \sigma}$ is the constraint $f$ grounded for the value assignments in the solution $\sigma$: 
$f_{\downarrow \sigma}: \prod_{x \in \setf{x}^{f}, \setf{x}^f \subseteq \var{\sigma}} \sigma(x) \to \mathbb{R} \cup \{\infty\}$, where $\sigma(x)$ is the value assigned to variable $x$ in the solution $\sigma$.
The cost $\setf{F}(\sigma) = \sum_{f \in \setf{F}} f_{\downarrow \sigma}(\sigma)$ is the sum of the costs across all the grounded constraints in $\sigma$. A solution $\sigma$ is \emph{complete} if $\var{\sigma} = \setf{X}$. Typically, the goal is to find an optimal complete solution $\sigma^* = \argmin_{\sigma, \var{\sigma} = \setf{X}} \setf{F}(\sigma)$.

\begin{figure*}[t]
\begin{center}
  \begin{minipage}[h]{30mm}
  \centering
 \includegraphics[height=0.85in]{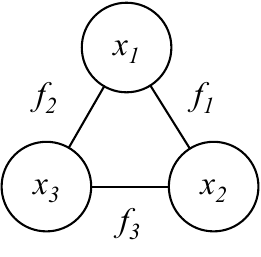}
  \end{minipage}
  \begin{minipage}[h]{23mm}
    \centering
    \begin{tabular}{|c|c||c|}
      \hline
      $x_1$ & $x_2$ & $f_1$ \\
      \hline
      \hline
      0 & 0 & 1 \\
      0 & 1 & 2 \\
      1 & 0 & 1 \\
      1 & 1 & 1 \\
      \hline
    \end{tabular}
  \end{minipage}
  \hspace{0.8em}
  \begin{minipage}[h]{23mm}
    \centering
    \begin{tabular}{|c|c||c|}
      \hline
      $x_1$ & $x_3$ & $f_2$ \\
      \hline
      \hline
      0 & 0 & 3 \\
      0 & 1 & 3 \\
      1 & 0 & 1 \\
      1 & 1 & 1 \\
      \hline
    \end{tabular}
  \end{minipage}
  \hspace{0.1in}
  \begin{minipage}[h]{23mm}
    \centering
    \begin{tabular}{|c|c||c|}
      \hline
      $x_2$ & $x_3$ & $f_3$ \\
      \hline
      \hline
      0 & 0 & 3 \\
      0 & 1 & 4 \\
      1 & 0 & 1 \\
      1 & 1 & 2 \\
      \hline
    \end{tabular}
 \end{minipage} 
\end{center}
\vspace{-1em}
\caption{Example DCOP. \label{fig:dcop}}
\end{figure*}

\begin{example}[DCOP]
\label{example1_dcop}
Figure~\ref{fig:dcop} illustrates a DCOP that we will use as a running example in this paper. It has three variables $x_1$, $x_2$, and $x_3$, where each variable $x_i$ is mapped to an agent $a_i$. All three variables have identical domains $\{0, 1\}$ for the two values they can take. There are three constraints: $f_1$ between variables $x_1$ and $x_2$, $f_2$ between variables $x_1$ and $x_3$, and $f_3$ between variables $x_2$ and $x_3$. The three tables in the figure tabulate the costs of the constraints. Each row of a table corresponds to a grounded constraint for the particular combination of value assignments to the variables in the scope of the constraint. The optimal complete solution is for both variables $x_1$ and $x_2$ to be assigned value 1 and for variable $x_3$ to be assigned value 0, incurring a total cost of 3, 1 from each constraint.
\end{example}

\xhdr{$k$-optimality:} We now briefly describe $k$-optimality as this concept is later used in the paper. A complete DCOP solution is $k$-optimal if no change in the value assignments of $k$ or fewer variables will improve the solution. $k$-optimal algorithms are algorithms that are guaranteed to find a $k$-optimal solution. Researchers have introduced a number of $k$-optimal algorithms as well as theoretical guarantees on the worst-case solution quality and upper bounds on the number of $k$-optimal solutions in a DCOP~\cite{pearce:07,BowringPPJT08,maheswaran:04b}.

\section{Explainable DCOP}

The \emph{Explainable DCOP} (X-DCOP) model extends the standard DCOP model to allow a user to ask an agent a contrastive query for a given solution of the DCOP. Specifically, an X-DCOP is defined by a tuple $\langle \mathcal{P}, \sigma, \mathcal{Q} \rangle$, where:
\squishlist
\item $\mathcal{P}$ is a DCOP instance.
\item $\sigma$ is a \emph{complete solution} to the DCOP $\mathcal{P}$.
\item $\mathcal{Q} = \langle a_Q, \sigma_Q, \hat{\sigma}_Q \rangle$ is a \emph{query} defined by a tuple, where $a_Q$ is the agent that is asked the query; $\sigma_Q \subseteq \sigma$ is the \emph{value assignment in the solution} $\sigma$ for a subset of variables $\var{\sigma_Q} \subseteq \var{\sigma}$; and $\hat{\sigma}_Q$ is an \emph{alternative value assignment} for that same subset of variables $\var{\sigma_Q} = \var{\hat{\sigma}_Q}$. In other words, for each variable $x \in \var{\sigma_Q}$, $\sigma_Q(x) \neq \hat{\sigma}_Q(x)$.
\squishend

\begin{example}[X-DCOP]
\label{example1}
Using the example DCOP from Figure~\ref{fig:dcop}, recall that $x_1$ is assigned value 1 in the optimal solution. A user may ask agent $a_1$ the following contrastive query: ``Why is $x_1$ assigned value $1$ in the solution instead of the alternative value~$0$?'' 
\end{example}

An \emph{X-DCOP solution} is a \emph{contrastive explanation} defined below. In essence, the explanation provides the grounded constraints that are affected by the change in value assignments in the query and the changes in costs if the alternative solution is adopted.

\begin{definition}[Contrastive Explanation]
\label{def:solution}
A \emph{contrastive explanation} for an X-DCOP $\langle \mathcal{P}, \sigma, \mathcal{Q} \rangle$ is defined by the tuple $\langle \setf{F}_{\downarrow \sigma_Q}, \setf{F}_{\downarrow \hat{\sigma}_Q}, \setf{F}_{\downarrow \sigma_Q}(\sigma_Q), \setf{F}_{\downarrow \hat{\sigma}_Q}(\hat{\sigma}_Q) \rangle$, where:
\squishlist
\item $\setf{F}_{\downarrow \sigma_Q}$ 
is the set of \emph{all} grounded constraints from the DCOP $\mathcal{P}$ whose scope includes a variable in the solution $\sigma_Q$ from the query $\mathcal{Q}$ and are grounded by the complete solution $\sigma$;
\item $\setf{F}_{\downarrow \hat{\sigma}_Q}$ 
is a \emph{possibly partial} set of grounded constraints from the DCOP $\mathcal{P}$ whose scope includes a variable in the solution $\hat{\sigma}_Q$ from the query $\mathcal{Q}$ and are grounded by the complete solution $(\sigma \setminus \sigma_Q) \cup \hat{\sigma}_Q$.
\item $\setf{F}_{\downarrow \sigma_Q}(\sigma_Q) = \sum_{f \in \setf{F}_{\downarrow \sigma_Q}} f(\sigma)$ is the cost of the solution $\sigma_Q$ from the query $\mathcal{Q}$ evaluated through the grounded constraints in $\setf{F}_{\downarrow \sigma_Q}$.
\item $\setf{F}_{\downarrow \hat{\sigma}_Q}(\hat{\sigma}_Q) = \sum_{f \in \setf{F}_{\downarrow \hat{\sigma}_Q}} f((\sigma \setminus \sigma_Q) \cup \hat{\sigma}_Q)$ is the cost of the alternative solution $\hat{\sigma}_Q$ from the query $\mathcal{Q}$ evaluated through the grounded constraints in $\setf{F}_{\downarrow \hat{\sigma}_Q}$.
\squishend
\end{definition}

\begin{example}[Contrastive Explanation]
\label{example2}
Continuing with Example~\ref{example1}, a possible contrastive explanation for the query is the following: ``Assigning $x_1$ its original value 1 will incur a cost of 2: 1 from constraint $f_1$ with $x_2$ and 1 from constraint $f_2$ with $x_3$. In contrast, assigning it the alternative value 0 will incur a \emph{larger} cost of 5: 2 from constraint $f_1$ with $x_2$ and 3 from constraint $f_2$ with $x_3$.'' Note that constraint $f_3$ is ignored in the explanation since the variables in its scope are not variables in the query. Thus, the cost of that constraint is identical for both the original and alternative value assignments $\sigma_Q$ and $\hat{\sigma}_Q$, respectively.
\end{example}

The definition of contrastive explanations for X-DCOPs allows some flexibility in the set of grounded constraints $\setf{F}_{\downarrow \hat{\sigma}_Q}$ as it allows some of the grounded constraints to be omitted from the set. As a result, it is possible to have a ``null'' explanation, in the cases where the set $\setf{F}_{\downarrow \hat{\sigma}_Q}$ is empty. Naturally, such an explanation is not particularly interesting because it is neither informative nor satisfying for users. Therefore, we propose that an explanation is \emph{valid} if it satisfies some core properties that provide the minimal amount of information for the explanation to be useful.


\begin{definition}[Validity]\label{def:validity}
A contrastive explanation 
$\langle \setf{F}_{\downarrow \sigma_Q}, \setf{F}_{\downarrow \hat{\sigma}_Q}, \setf{F}_{\downarrow \sigma_Q}(\sigma_Q), \setf{F}_{\downarrow \hat{\sigma}_Q}(\hat{\sigma}_Q) \rangle$
is \emph{valid} iff $\setf{F}_{\downarrow \sigma_Q}(\sigma_Q) \leq \setf{F}_{\downarrow \hat{\sigma}_Q}(\hat{\sigma}_Q)$.
\end{definition}

Given this definition of validity, it is interesting to observe that it is not guaranteed that valid explanations exist for all X-DCOPs. For example, if the complete solution $\sigma$ in the X-DCOP is a random value assignment to all variables, then it may be the case that having some variables assigned values from the alternative assignment $\hat{\sigma}_Q$ will result in a smaller cost and is thus a better solution, violating the second validity property. Therefore, the existence of valid explanations is strongly based on the properties of the complete solution $\sigma$ and the query $\mathcal{Q}$ in the X-DCOP. We show in the following that valid explanations always exist for a specific case.

\begin{theorem}
For an X-DCOP $\langle \mathcal{P}, \sigma, \mathcal{Q} \rangle$, if $\sigma$ is a $k$-optimal solution and the number of variables in the query is not greater than $k$, that is, $|\var{\sigma_Q}| = |\var{\hat{\sigma}_Q}| \leq k$, then a valid explanation exists. 
\label{theorem:existence}
\end{theorem}
\begin{proof} (Sketch)
Let $\setf{F}_{\downarrow \sigma_Q}$ and $\setf{F}_{\downarrow \hat{\sigma}_Q}$ denote the sets of \emph{all} grounded constraints whose scope includes a variable in the solutions $\sigma_Q$ and $\hat{\sigma}_Q$, respectively, and are grounded by the complete solutions $\sigma$ and $(\sigma \setminus \sigma_Q) \cup \hat{\sigma}_Q$, respectively.
Assume that $\sigma$ is a $k$-optimal solution based on the premise of the lemma. Then, by definition, it is guaranteed that no change in value assignments to any subset of $k'$ variables, with $k' \leq k$, will result in a better solution (i.e., one with a smaller cost)~\cite{pearce:07}. Consequently, for any query $\langle a_Q, \sigma_Q, \hat{\sigma}_Q \rangle$ where $|\var{\sigma_Q}| = |\var{\hat{\sigma}_Q}| \leq k$, it must be the case that $\setf{F}_{\downarrow \sigma_Q}(\sigma_Q) \leq \setf{F}_{\downarrow \hat{\sigma}_Q}(\hat{\sigma}_Q)$.
\end{proof}

Finally, while the proof above assumes that $\setf{F}_{\downarrow \hat{\sigma}_Q}$ is the set of \emph{all} grounded constraints whose scope includes a variable in the solutions $\hat{\sigma}_Q$ to prove the existence of a valid explanation, it may often be the case that explanations with only a partial subset of constraints are also valid, illustrated below.

\begin{example}[Alternative Contrastive Explanation]
\label{example3}
Continuing from Example~\ref{example2}, an alternative contrastive explanation is the following: ``Assigning $x_1$ its original value 1 will incur a cost of 2: 1 from constraint $f_1$ with $x_2$ and 1 from constraint $f_2$ with $x_3$. In contrast, assigning it the alternative value 0 will incur a \emph{larger} cost of \emph{at least 3} from constraint $f_2$ with $x_3$.'' Note that this explanation omits constraint $f_1$ with $x_2$, which is unnecessary since the cost of $f_2$ with $x_3$ alone is already greater than the cost in the original solution $\sigma$.
\end{example}


\section{X-DCOP Framework: CEDAR}\label{sec:framework}

\begin{algorithm}[!t]
\DontPrintSemicolon
\KwIn{$\sigma$, $\mathcal{Q} = \langle a_Q$, $\sigma_Q$, $\hat{\sigma}_Q \rangle$}
\KwResult{$\langle \setf{F}_{\downarrow \sigma_Q}, \setf{F}_{\downarrow \hat{\sigma}_Q}, \setf{F}_{\downarrow \sigma_Q}(\sigma_Q), \setf{F}_{\downarrow \hat{\sigma}_Q}(\hat{\sigma}_Q) \rangle$}

$\setf{F}^{a_Q}_{\downarrow \sigma_Q} \gets$ getOwnGroundedConstraints($\sigma_Q$) \\ 
$\setf{F}^{a_Q}_{\downarrow \hat{\sigma}_Q} \gets$ getOwnGroundedConstraints($\hat{\sigma}_Q$) \\ 
$\setf{F}^{\lnot a_Q}_{\downarrow \sigma_Q} \gets \setf{F}^{\lnot a_Q}_{\downarrow \hat{\sigma}_Q} \gets \emptyset$ \\
\ForEach{$a_i$ such that $\exists x: \alpha(x) = a_i, x \in \var{\sigma_Q}$}{
    Send REQUEST($a_Q$, $\sigma_Q$) to agent $a_i$ \\
    Send REQUEST($a_Q$, $\hat{\sigma}_Q$) to agent $a_i$ \\
}
Wait until all REPLY messages are received \\
$\setf{F}_{\downarrow \sigma_Q} \gets \setf{F}^{a_Q}_{\downarrow \sigma_Q} \cup \setf{F}^{\lnot a_Q}_{\downarrow \sigma_Q}$ \\
$\setf{F}_{\downarrow \hat{\sigma}_Q} \gets \setf{F}^{a_Q}_{\downarrow \hat{\sigma}_Q} \cup \setf{F}^{\lnot a_Q}_{\downarrow \hat{\sigma}_Q}$ \\
$\setf{F}_{\downarrow \sigma_Q}(\sigma_Q) \gets \sum_{f \in \setf{F}_{\downarrow \sigma_Q}} f(\sigma)$ \\
$\setf{F}_{\downarrow \hat{\sigma}_Q}(\hat{\sigma}_Q) \gets \sum_{f \in \setf{F}_{\downarrow \hat{\sigma}_Q}} f((\sigma \setminus \sigma_Q) \cup \hat{\sigma}_Q)$ \\
\Return $\langle \setf{F}_{\downarrow \sigma_Q}, \setf{F}_{\downarrow \hat{\sigma}_Q}, \setf{F}_{\downarrow \sigma_Q}(\sigma_Q), \setf{F}_{\downarrow \hat{\sigma}_Q}(\hat{\sigma}_Q) \rangle$

\caption{CEDAR}
\label{alg}
\end{algorithm}

\begin{procedure}[!t]
\DontPrintSemicolon

$\setf{F}^{a_i}_{\downarrow \bar{\sigma}_Q} \gets$ getOwnGroundedConstraints($\bar{\sigma}_Q$) \\
Send REPLY($a_i$, $\bar{\sigma}_Q$, $\setf{F}^{a_i}_{\downarrow \bar{\sigma}_Q}$) to sender agent $a_Q$

\caption{When Receive REQUEST($a_Q$, $\bar{\sigma}_Q$)}
\end{procedure}

\begin{procedure}[!t]
\DontPrintSemicolon
\uIf{$\bar{\sigma}_Q = \sigma_Q$}{
$\setf{F}^{\lnot a_Q}_{\downarrow \sigma_Q} \gets \setf{F}^{\lnot a_Q}_{\downarrow \sigma_Q} \cup \setf{F}^{a_i}_{\downarrow \bar{\sigma}_Q}$
}
\Else{
$\setf{F}^{\lnot a_Q}_{\downarrow \hat{\sigma}_Q} \gets \setf{F}^{\lnot a_Q}_{\downarrow \hat{\sigma}_Q} \cup \setf{F}^{a_i}_{\downarrow \bar{\sigma}_Q}$
}

\caption{When Receive REPLY($a_i$, $\bar{\sigma}_Q$, $\setf{F}^{a_i}_{\downarrow \bar{\sigma}_Q}$)}
\end{procedure}



We now describe \emph{Constraint-based Explanation via Distributed Algorithms and Reasoning} (CEDAR), our framework for solving X-DCOPs. We first start with a fairly naive and straightforward framework (pseudocode presented in Algorithm~\ref{alg}), before describing optimizations and variants that we can also consider. At a high level, when an agent $a_Q$ receives a query $\mathcal{Q} = \langle a_Q$, $\sigma_Q$, $\hat{\sigma}_Q \rangle$, it first identifies all of its own grounded constraints that are relevant to the query through the getOwnGroundedConstraints($\cdot$) function for both the original and alternative solutions $\sigma_Q$ and $\hat{\sigma}_Q$, respectively. Those grounded constraints are stored in their respective sets $\setf{F}^{a_Q}_{\downarrow \sigma_Q}$ and $\setf{F}^{a_Q}_{\downarrow \hat{\sigma}_Q}$ (Lines~1-2). Specifically, the getOwnGroundedConstraints($\cdot$) function accepts a solution $\bar{\sigma}_Q$ (that is either the original or alternative solution) as its argument and returns exactly all the constraints whose scope includes a variable $x$ associated with the agent $a_Q$ (i.e.,~$\alpha(x) = a_Q$) and is also a variable in the query (i.e.,~$x \in \var{\bar{\sigma}_Q}$) grounded with $(\sigma \setminus \sigma_{Q}) \cup \bar{\sigma}_Q$. 

Next, it identifies all of the other grounded constraints that it is unaware of (because the scope of those constraints do not include any of its variables). To do so, it sends a REQUEST message to each agent whose variables are in the query (Lines~4-6). When an agent receives such a REQUEST message, it also identifies all its own grounded constraints for the query and returns them to the sender (Lines~13-14).

Once agent $a_Q$ receives all the grounded constraints from all the other agents, it stores all those constraints in the appropriate sets $\setf{F}^{\lnot a_Q}_{\downarrow \sigma_Q}$ or $\setf{F}^{\lnot a_Q}_{\downarrow \hat{\sigma}_Q}$ (Lines~15-18), which are then combined with its own grounded constraints in the respective sets $\setf{F}_{\downarrow \sigma_Q}$ and $\setf{F}_{\downarrow \sigma_Q}$ (Lines~8-9). These sets then contain all grounded constraints whose scope includes a variable in the query. Finally, it evaluates the costs $\setf{F}_{\downarrow \sigma_Q}(\sigma_Q)$ and $\setf{F}_{\downarrow \hat{\sigma}_Q}(\hat{\sigma}_Q)$ of the solution $\sigma_Q$ and the alternative $\hat{\sigma}_Q$, respectively, and returns them (Lines~10-12).

CEDAR is guaranteed to terminate and will return a valid explanation if one exists (see Theorems~\ref{th:completeness} and~\ref{th:correctness}).



\subsection{Optimizations}
\label{sec:XDCOP-opt}
We now describe two optimizations O1 and O2 that can be applied within CEDAR to improve it. 



\xhdr{Optimization O1 (Returning Shortest Explanations):} For an explanation to be valid, it is required that $\setf{F}_{\downarrow \hat{\sigma}_Q}(\hat{\sigma}_Q) \geq \setf{F}_{\downarrow \sigma_Q}(\sigma_Q)$. However, as illustrated in Example~\ref{example3}, it can be the case that there exist \emph{smaller} subsets of grounded constraints $\setf{F}'_{\downarrow \hat{\sigma}_Q} \subseteq \setf{F}_{\downarrow \hat{\sigma}_Q}$ whose cost $\setf{F}'_{\downarrow \hat{\sigma}_Q}(\hat{\sigma}_Q) \geq \setf{F}_{\downarrow \sigma_Q}(\sigma_Q)$ is also greater than or equal to the cost of the original solution. Motivated by our user study, which shows that human users prefer shorter explanations (see Section~\ref{sec:user-study}), our first optimization seeks to find a minimal subset of grounded constraints $\setf{F}^*_{\downarrow \hat{\sigma}_Q}$ for the alternative solution $\hat{\sigma}_Q$ that forms a valid explanation.

To compose this minimal subset, CEDAR sorts the grounded constraints in $\setf{F}_{\downarrow \hat{\sigma}_Q}$ in decreasing cost order. Then, the minimal subset $\setf{F}^*_{\downarrow \hat{\sigma}_Q}$ is constructed by incrementally adding the remaining grounded constraint with the highest cost to the set until either (1)~$\setf{F}^*_{\downarrow \hat{\sigma}_Q}(\hat{\sigma}_Q) \geq \setf{F}_{\downarrow \sigma_Q}(\sigma_Q)$, in which case a valid explanation is found; or (2)~$\setf{F}^*_{\downarrow \hat{\sigma}_Q}(\hat{\sigma}_Q) < \setf{F}_{\downarrow \sigma_Q}(\sigma_Q)$ even after adding all the grounded constraints in $\setf{F}_{\downarrow \hat{\sigma}_Q}$, in which case a valid explanation does not exist. 

\xhdr{Optimization O2 (Decentralized Parallel Sorting):} We exploit the fact that the grounded constraints can be sorted in parallel by the different agents before the agents send their grounded constraints to $a_Q$. Once agent $a_Q$ receives the sorted constraints from each agent, it identifies the constraint with the highest cost from among the heads of the sorted lists. To efficiently do this, it utilizes a priority heap that stores all the heads. Once that constraint is identified, the agent $a_Q$ removes that constraint and adds it to the set $\setf{F}^*_Q$. It repeats this process until either (1)~the sum of the costs of the grounded constraints $\setf{F}^*_{\downarrow \hat{\sigma}_Q}(\hat{\sigma}_Q) \geq \setf{F}_{\downarrow \sigma_Q}(\sigma_Q)$, at which point it found a valid explanation; or (2)~$\setf{F}^*_{\downarrow \hat{\sigma}_Q}(\hat{\sigma}_Q) < \setf{F}_{\downarrow \sigma_Q}(\sigma_Q)$ even after adding all the grounded constraints of all relevant agents, in which case a valid explanation does not exist. 


This sorting process will take $O(|\mathcal{F}_a| \log(|\mathcal{F}_a|))$ for each agent $a$, where $\mathcal{F}_a$ is the set of grounded constraints of the agent relevant to the query. Building the initial priority heap will take $O(|\mathcal{A}_Q|)$, where $\mathcal{A}_Q$ is the set of agents whose variables are in the query. Identifying the constraint with the highest cost is then just $O(1)$, but inserting the next head of the list into the priority heap will take $O(\log(|\mathcal{A}_Q|))$. Since we will have at most $O(|\mathcal{F}_a| \cdot |\mathcal{A}_Q|)$ insertions into the priority heap, the overall runtime is thus $O(|\mathcal{F}_a| \log(|\mathcal{F}_a|))$ + $O(|\mathcal{F}_a| \cdot |\mathcal{A}_Q| \cdot \log(|\mathcal{A}_Q|))$. Assuming that $\log(|\mathcal{F}_a|) < |\mathcal{A}_Q| \cdot \log(|\mathcal{A}_Q|)$, the runtime of this optimization can be simplified to $O(|\mathcal{F}_a| \cdot |\mathcal{A}_Q| \cdot \log(|\mathcal{A}_Q|))$, which is smaller than the runtime for the centralized approach, which is $O(|\mathcal{F}_a| \cdot |\mathcal{A}_Q| \cdot \log(|\mathcal{F}_a| \cdot |\mathcal{A}_Q|))$.



\subsection{Suboptimal Variants}

We now describe two subvariants V1 and V2 for CEDAR that trade off explanation lengths for shorter runtimes. 

\xhdr{Suboptimal Variant V1:} The runtime of optimization O2 is dominated by the priority heap operations to identify the grounded constraint with the highest cost. In this variant, we speed up this process at the cost of losing the guarantee that the constraint with the highest cost will be returned. Specifically, agent $a_Q$ merges the sorted lists of each of the other agents into a single pseudo-sorted list; the first set of $|\mathcal{A}_Q|$ elements of the pseudo-sorted list are the first elements from each of the $\mathcal{A}_Q$ sorted lists of the other agents; the second set of $|\mathcal{A}_Q|$ elements are the second elements from each of the sorted lists; and so on. The agent $a_Q$ iteratively adds the grounded constraints from the pseudo-sorted list into a set $\setf{F}'_{\downarrow \hat{\sigma}_Q}$ until either (1)~$\setf{F}'_{\downarrow \hat{\sigma}_Q}(\hat{\sigma}_Q) \geq \setf{F}_{\downarrow \sigma_Q}(\sigma_Q)$, in which case a valid explanation is found; or (2)~$\setf{F}'_{\downarrow \hat{\sigma}_Q}(\hat{\sigma}_Q) < \setf{F}_{\downarrow \sigma_Q}(\sigma_Q)$ even after adding all grounded constraints from the pseudo-sorted list, in which case a valid explanation does not exist. 


\xhdr{Suboptimal Variant V2:} While optimizations O1 and O2 aim to find shortest explanations, they can come at a high runtime and communication cost as \emph{all the relevant grounded constraints} in $\setf{F}_{\downarrow \hat{\sigma}_Q}$ for the alternative solution $\hat{\sigma}_Q$ need to be considered and communicated to agent $a_Q$. In this variant, we consider an anytime extension whereby agent $a_Q$ sends REQUEST messages to only a subset of agents to request for their relevant grounded constraints with the hope that the grounded constraints received, together with its own relevant grounded constraints, are sufficient to form a valid explanation. If the grounded constraints received are insufficient, additional REQUEST messages are sent to other agents. This process repeats until either a valid explanation can be formed or grounded constraints are received from all agents. Although this approach may return a valid explanation faster and with fewer messages compared CEDAR with Optimizations~1 and~2, note that it is no longer guaranteed to return explanations that are shortest.

To identify which subset of agents to send REQUEST messages to, agent $a_Q$ first includes all its own grounded constraints, which are relevant for the alternative solution $\hat{\sigma}_Q$ in the set $\setf{F}'_{\downarrow \hat{\sigma}_Q}$ and computes the cost $\setf{F}'_{\downarrow \hat{\sigma}_Q}(\hat{\sigma}_Q)$ of the alternative solution $\hat{\sigma}_Q$ with those constraints. If $\setf{F}'_{\downarrow \hat{\sigma}_Q}(\hat{\sigma}_Q) \geq \setf{F}_{\downarrow \sigma_Q}(\sigma_Q)$, then a valid explanation has been found, including only the constraints of agent $a_Q$. Otherwise, the agent must identify relevant grounded constraints from other agents to add to the set $\setf{F}'_{\downarrow \hat{\sigma}_Q}$ so that the inequality is met. 

The difference in costs $\Delta = \setf{F}_{\downarrow \sigma_Q}(\sigma_Q) - \setf{F}'_{\downarrow \hat{\sigma}_Q}(\hat{\sigma}_Q)$ reflects the additional cost that agent $a_Q$ needs to explain through other relevant grounded constraints. To do this, it sorts the other agents according to their degree $|\mathcal{F}_a|$, where $|\mathcal{F}_a|$ is the number of relevant grounded constraints of agent $a$ for the alternative solution $\hat{\sigma}_Q$.\footnote{Agent $a_Q$ has this information after requesting for and receiving \emph{all} the relevant grounded constraints of each agent for the current solution $\sigma_Q$ because the numbers of relevant grounded constraints for both $\sigma_Q$ and the alternative $\hat{\sigma}_Q$ are identical.} Then, assuming that the maximum \emph{finite} cost of grounded constraints is $c_{\max}$, agent $a_Q$ will send REQUEST messages to the top $p$ agents with the largest degrees, such that the sum of their estimated constraint costs $\sum_{a, |a| \leq p} |\mathcal{F}_a| \cdot c_{\max} \geq \Delta$ is at least the difference in costs. As the estimated constraint costs may be overly optimistic since we are using the maximum finite cost as the estimate, it is likely that the actual constraint costs may not make up the difference. In this case, the agent repeats the process by computing an updated $\Delta$, taking into account the updated costs $\setf{F}'_{\downarrow \hat{\sigma}_Q}(\hat{\sigma}_Q)$ with the newly received constraints added, and identifies the next subset of agents to send REQUEST messages to. This process continues until either (1)~$\setf{F}'_{\downarrow \hat{\sigma}_Q}(\hat{\sigma}_Q) \geq \setf{F}_{\downarrow \sigma_Q}(\sigma_Q)$, in which case a valid explanation is found; or (2)~$\setf{F}'_{\downarrow \hat{\sigma}_Q}(\hat{\sigma}_Q) < \setf{F}_{\downarrow \sigma_Q}(\sigma_Q)$ even after adding all grounded constraints from all relevant agents, in which case a valid explanation does not exist.

\subsection{Theoretical Results}

We assume that messages are never lost and are received in the order that they are sent -- a standard assumption in the DCOP literature~\cite{fioretto:18}.

\begin{theorem}
CEDAR and all its optimized versions and suboptimal variants are guaranteed to terminate.
\label{th:completeness}
\end{theorem}
\begin{proof} (Sketch)
The main CEDAR framework illustrated in Algorithm~\ref{alg} is guaranteed to terminate since the agent will receive REPLY messages from all agents that it sent REQUEST messages to, under the assumption that messages are never lost. Optimizations O1 and O2 as well as Variants V1 and V2 do not affect the termination guarantee because, in the worst case, they will add all relevant grounded constraints into either the set $\setf{F}^*_{\downarrow \hat{\sigma}_Q}$ for O1 and O2 or the set $\setf{F}'_{\downarrow \hat{\sigma}_Q}$ for V1 and V2, after which they terminate. 
\end{proof}

\begin{theorem}
CEDAR and all its optimized versions and suboptimal variants are guaranteed to return a valid explanation if one exists.
\label{th:correctness}
\end{theorem}
\begin{proof} (Sketch)
The main CEDAR framework illustrated in Algorithm~\ref{alg} computes $\setf{F}_{\downarrow \sigma_Q}$, $\setf{F}_{\downarrow \hat{\sigma}_Q}$, $\setf{F}_{\downarrow \sigma_Q}(\sigma_Q)$, and $\setf{F}_{\downarrow \hat{\sigma}_Q}(\hat{\sigma}_Q)$ exactly as defined in Definition~\ref{def:solution}. Thus, if a valid explanation exists, CEDAR will return it. Optimizations O1 and O2 as well as Variants V1 and V2 do not affect the correctness because they either return a valid explanation after finding it or, in the worst case, guarantee that no valid explanation exists after adding all relevant grounded constraints.
\end{proof}

\section{Computational Experimental Evaluation}

We now describe our computational evaluations of CEDAR.\footnote{Simulation source code:  \url{https://github.com/benrachmut/ExplainableDCOP}} To demonstrate the generality of CEDAR, we evaluate it on two DCOP types, commonly used in the literature:

\squishlist
\item \textbf{Random Graphs:} Unstructured, abstract, constraint graph topologies with density $p_1 = \{0.2, 0.7 \}$. Each variable has $10$ values in its domain. The constraint costs were uniformly selected between $1$ and $100$. 

\item \textbf{Meeting Scheduling:} 
Using the EAV formulation~\cite{maheswaran:04a}, meetings are represented by variables and time slots form their domain. 
Each user $u_i$ has a preferred time slot $d_i^{*}$, chosen randomly from a uniform distribution. The cost for a user attending a meeting in time slot $d_j$ is $2^{|d_j - d_i^{*}|}$. Binary constraints arise when a user needs to attend two meetings, with equal time slots incurring an infinite cost\footnote{We used $10,\!000$ as a proxy for infinite cost in our experiments.} and non-equal time slots incurring a cost that reflects the preferences of all users of the meetings. 
Instances were created where each user attends two meetings and the number of users were varied to maintain a graph density $p_1$ of $0.5$. 
\squishend
To investigate the impact of the properties of DCOP solutions, we evaluated CEDAR with both \textbf{optimal} complete solutions and \textbf{1-opt} solutions. To construct queries with respect to a solution $\sigma$, we randomly selected the agent $a_Q$ receiving the query and the variables $\var{\sigma_Q}$ queried. These variables are also constrained with at least another variable in that set. We consider two approaches to select alternative values for the variables queried:
\squishlist
\item {\bf Random Baseline:} In experiments involving both optimal and 1-opt solutions, for each variable $x \in \var{\sigma_Q}$, an alternative value $d \in D_{x} \setminus \{d^*, d^{1}\}$ is randomly selected, where $d^*$ and $d^1$ are the values of $x$ in the optimal and 1-opt solutions, respectively. 
In experiments involving only optimal solutions, an alternative value $d \in D_{x} \setminus \{d^*\}$ is randomly selected.

\item {\bf Best Alternative:} Inspired by the fact that queries from human users typically involve alternative values that they think are best, this approach identifies the best alternative value to include in queries. To do this, we solved an alternate DCOP, where the domain of each variable $x' \notin \var{\sigma_Q}$ is exactly its value in the solution $\sigma$ and the domain of each variable $x \in \var{\sigma_Q}$ is its original domain $D_x$ without its value in the solution $\sigma$. 
\squishend
Finally, to measure distributed runtimes, we use \emph{Non-Concurrent Logic Operations} (NCLOs)~\cite{ZivanM06b,NetzerGM12}, which avoid double-counting concurrent computations of agents in a distributed environment.

\begin{figure}[t]
\centering
\small

\includegraphics[width=1\columnwidth]{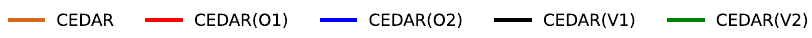}  

\begin{subfigure}[t]{1\columnwidth}
\centering
\includegraphics[height=15em]{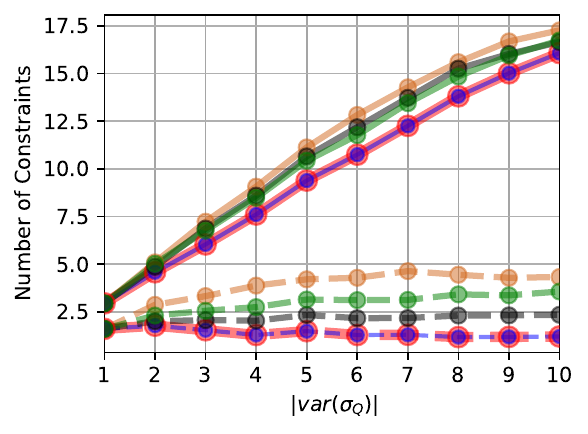} \\ 
\vspace{-0.5em} (a) Number of grounded constraints in the explanations as a function of the number of query variables $|\var{\sigma_Q}|$.
\end{subfigure} \\
\begin{subfigure}[t]{1\columnwidth}
\centering
\includegraphics[height=15em]{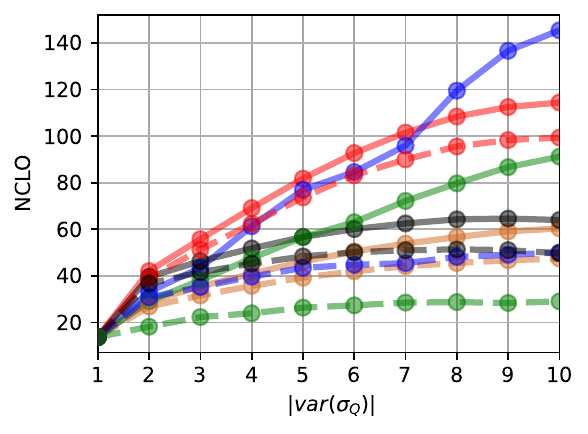} 
\\ 
\vspace{-0.5em} (b) NCLOs as a function of number of query variables $|\var{\sigma_Q}|$
\end{subfigure}

\caption{Experimental results for CEDAR, its optimized versions, and its variants on meeting scheduling problems with $10$ meetings for optimal complete solutions; solid lines denote results for best alternative queries and dashed lines denote results for random queries.}
\label{fig:results}
\end{figure}



\begin{table}[t]
\centering \footnotesize
\begin{minipage}{0.30\textwidth} 
\centering
\begin{tabular}{|c|cc|}
\hline
\multirow{2}{*}{$|\var{\sigma_Q}|$} & \multicolumn{2}{c|}{\textbf{Best Alternative}}            \\ 
                                   & \multicolumn{1}{c|}{\textbf{1-Opt}}   & \textbf{Optimal}   \\ \hline 
1                                  & \multicolumn{1}{c|}{100\%}     & 100\%        \\ 
3                                  & \multicolumn{1}{c|}{86\%}      & 100\%        \\ 
5                                  & \multicolumn{1}{c|}{79\%}      & 100\%        \\ 
7                                  & \multicolumn{1}{c|}{78\%}      & 100\%        \\ 
9                                  & \multicolumn{1}{c|}{81\%}      & 100\%        \\ \hline \hline
\multirow{2}{*}{$|\var{\sigma_Q}|$} & \multicolumn{2}{c|}{\textbf{Random}}               \\ 
                                   & \multicolumn{1}{c|}{\textbf{1-Opt}}   & \textbf{Optimal}   \\ \hline 
1                                  & \multicolumn{1}{c|}{100\%}      & 100\%        \\ 
3                                  & \multicolumn{1}{c|}{100\%}      & 100\%        \\ 
5                                  & \multicolumn{1}{c|}{100\%}      & 100\%        \\ 
7                                  & \multicolumn{1}{c|}{100\%}      & 100\%        \\
9                                  & \multicolumn{1}{c|}{100\%}      & 100\%        \\ \hline
\end{tabular}
\end{minipage}%
\hspace{1em} 
\begin{minipage}{0.30\textwidth} 
\centering
\begin{tabular}{|c|cc|}
\hline
\multirow{2}{*}{$|\setf{A}|$} & \multicolumn{2}{c|}{\textbf{Best Alternative}}            \\ 
                              & \multicolumn{1}{c|}{\textbf{1-Opt}}   & \textbf{Optimal}   \\ \hline
10                            & \multicolumn{1}{c|}{87\%}      & 100\%        \\
20                            & \multicolumn{1}{c|}{93\%}      & -          \\
30                            & \multicolumn{1}{c|}{100\%}      & -          \\
40                            & \multicolumn{1}{c|}{99\%}      & -          \\
50                            & \multicolumn{1}{c|}{100\%}      & -          \\ \hline \hline
\multirow{2}{*}{$|\setf{A}|$} & \multicolumn{2}{c|}{\textbf{Random}}               \\ 
                              & \multicolumn{1}{c|}{\textbf{1-Opt}}   & \textbf{Optimal}   \\ \hline
10                            & \multicolumn{1}{c|}{100\%}     & 100\%        \\
20                            & \multicolumn{1}{c|}{100\%}     & -          \\
30                            & \multicolumn{1}{c|}{100\%}     & -          \\
40                            & \multicolumn{1}{c|}{100\%}     & -          \\
50                            & \multicolumn{1}{c|}{100\%}     & -          \\ \hline
\end{tabular}

\end{minipage}

\caption{Percentage of instances for which valid explanations were found as a function of query variables $|\var{\sigma_Q}|$ (left, with $|\mathcal{A}| = 10$) and agents $|\mathcal{A}|$ (right, with $|\var{\sigma_Q}| = 5$), with best alternative (top) and random queries (bottom) on 1-opt and optimal solutions.}

\label{tab:results}
\end{table}

Table~\ref{tab:results} and Figure~\ref{fig:results} show some representative results for meeting scheduling problems, where each data point is an average over 100 X-DCOP instances.\footnote{See appendix for full set of computational results.} Table~\ref{tab:results} shows the percentage of instances for which CEDAR found valid explanations for different parameter configurations. We highlight the following observations:
\squishlist
\item The results demonstrate the scalability of CEDAR as it is able to scale to large problems with up to 50 agents. Results are unavailable for X-DCOPs with optimal DCOP solutions for 20 or more agents because such optimal solutions were not found due to scalability issues of existing complete DCOP solvers. 
\item CEDAR finds valid explanations for all X-DCOP instances with 1-opt solutions and $\var{\sigma_Q} = 1$. However, it fails to find valid explanations for some instances with queries that have more variables and best alternative values.
\item Interestingly, for 1-opt solutions, valid explanations are found for all instances with random alternative values, even for queries with $|\var{\sigma_Q}| > 1$. Therefore, even though valid explanations are not guaranteed to exist for such queries, in practice, they are quite common. 
\squishend

Figure~\ref{fig:results} presents the results of CEDAR, its two optimized versions, and its two suboptimal variants on X-DCOPs with complete optimal solutions. Since we are evaluating results where valid explanations are guaranteed to exist, we omit results for 1-opt solutions. The solid lines in the figure denote results for the best alternative queries, and the dashed lines denote results for random queries. 

Figure~\ref{fig:results}(a) shows that the lengths of valid explanations tend to increase as the queries increase in complexity in terms of the number of query variables $|\var{\sigma_Q}|$, especially for the best alternative queries. Within each type of query, CEDAR finds the longest explanations since it includes the full set of grounded constraints $\setf{F}_{\downarrow \hat{\sigma}_Q}$ for the alternative solution $\hat{\sigma}_Q$ in its explanations. CEDAR with optimizations O1 and O2 find shortest explanations, as expected since they optimize for explanation length. CEDAR with suboptimal variants V1 and V2 find longer explanations, but they are generally shorter than CEDAR's.

Figure~\ref{fig:results}(b) shows that the runtimes increase as the complexity of the queries increases, where the runtimes for the best alternative queries are greater than the runtimes for random queries. Within each query type, CEDAR is the fastest. CEDAR with optimizations O1 and O2 are the slowest as they sort grounded constraints to identify shortest explanations. For random queries, O2 is consistently faster than O1. However, for best alternative queries, this is only true when $|\var{\sigma_Q}|$ is small because $\log(|\mathcal{F}_a|) < |\mathcal{A}_Q| \cdot \log(|\mathcal{A}_Q|)$ (see the description of O2 in Section~\ref{sec:XDCOP-opt}). As $|\var{\sigma_Q}|$ increases, $\log(|\mathcal{F}_a|)$ increases, and O1 loses its speed-up compared to O2. Finally, CEDAR with variants V1 and V2 have smaller runtimes, but they are still larger than CEDAR's. Therefore, they show the empirical tradeoffs made by CEDAR with V1 and V2 between solution quality in terms of explanation length and runtime in terms of NCLOs.

\section{Human User Study}
\label{sec:user-study}

We conducted a human user study to examine whether and how explanation length affects user comprehension and satisfaction in DCOP-based systems. We used meeting scheduling problems, as it is an application domain that users are likely to be very familiar with.
We now briefly describe our study.\footnote{See appendix for complete study materials, including the full text of all explanations and questions as well as user demographic information.}

\xhdr{Study Design:} We created a synthetic meeting scheduling scenario involving four participants (the user and three others -- Bob, Charlie, and David) coordinating two meetings (M1 and M2) over four possible time slots: Morning, noon, afternoon, and evening. Bob and the user needed to attend both meetings, while Charlie and David needed to attend only one.

The scheduler proposed holding M1 in the afternoon and M2 in the evening. Participants were presented with a contrastive query that compared this schedule with an alternative where M1 is at noon and M2 is in the afternoon. To evaluate how explanation length affects user comprehension and preferences, we created two groups that were shown identical scenarios but received different pairs of explanations: 
\squishlist
    \item \textit{Group 1} compared explanations of length two versus three.
    \item \textit{Group 2} compared explanations of length two versus four.
\squishend

Here, the length of an explanation refers to the number of agents' preferences used to justify the scheduling decision.


For each pair of explanations, participants were asked two key questions: (Q1)~Which explanation was easier to understand; and (Q2)~Which explanation they would be more likely to provide when explaining the scheduling decision to others. 

Our primary hypothesis is as follows:

\begin{quote}
    ``\emph{Participants would prefer and better understand shorter explanations compared to longer ones.}''
\end{quote}

\xhdr{Study Results and Discussion:} We recruited 130 participants through Prolific~\cite{palan2018prolific}, randomly assigning 65 participants to each group.\footnote{The study was approved by our institution's ethics board and adhered to the guidelines for responsible research practices.} All participants were proficient in English, and were compensated with \$2.50. To ensure data quality, we included two attention check questions that verify the understanding of the scenario by the participants. We removed participants who failed both attention checks (one from Group~1 and two from Group~2). To evaluate our hypothesis, we set our significance level at $\alpha = 0.05$.

\begin{table}[!t]
\centering \footnotesize
\begin{tabular}{|l|ll|ll|}
\hline
\multirow{2}{*}{\textbf{Measure}} & \multicolumn{2}{c|}{\textbf{Group 1}} & \multicolumn{2}{c|}{\textbf{Group 2}} \\
 & \multicolumn{2}{c|}{(expl. length 2 vs 3)} & \multicolumn{2}{c|}{(expl. length 2 vs 4)} \\
\hline
Q1: \textit{Easier to understand}
 & \multicolumn{2}{c|}{41 (64.1\%) prefer 2}
 & \multicolumn{2}{c|}{51 (81.0\%) prefer 2}\\

Q2: \textit{Would provide to others} 
 & \multicolumn{2}{c|}{41 (64.1\%) prefer 2}
 & \multicolumn{2}{c|}{49 (77.8\%) prefer 2} \\
\hdashline

$\chi^2$ statistic   
 & Q1: 5.06 & Q2: 5.06
 & Q1: 24.14 & Q2: 19.44 \\

$p$-value              
 & Q1: 0.024 & Q2: 0.024
 & Q1: $<$ 0.001 & Q2: $<$ 0.001 \\

Effect size (Cramer's V) 
 & Q1: 0.281 & Q2: 0.281 
 & Q1: 0.619 & Q2: 0.556 \\
\hline
\end{tabular}%
\caption{User study results showing the number of participants who favored the shorter explanation, along with the results of the statistical analysis for $\alpha = 0.05$.}
\label{tab:study-results}
\end{table}

Table~\ref{tab:study-results} presents the results of our analysis, which strongly support our hypothesis. In both experimental conditions, participants showed a clear preference for shorter explanations. In Group 1, 64.1\% of participants found the shorter explanation easier to understand and were more likely to use it when explaining to others ($\chi^2 \!=\! 5.06$ and $p \!=\! 0.024$ for both Q1 and Q2), with a medium effect size (Cramer's V $\!=\! 0.281$ for both Q1 and Q2). This preference was substantially stronger in Group 2, where the comparison was between explanations of larger length differences. Here, 81.0\% found the shorter explanation easier to understand, and 77.8\% preferred to provide it to others ($\chi^2 \!=\! 24.14$ and $19.44$ for Q1 and Q2, respectively, $p \!<\! 0.001$ for both), with large effect sizes (Cramer's V $\!=\! 0.619$ and $0.556$ for Q1 and Q2, respectively). 

These findings suggest that users strongly prefer concise explanations and this preference becomes more pronounced with increasing difference in explanation lengths. Thus, these results directly support our attempts to design optimization techniques for finding the shortest valid explanations.

\section{Related Work}

The grounded constraints in the explanations of X-DCOP bear some similarity to \emph{nogoods} in Constraint Satisfaction Problems (CSPs) and Constraint Optimization Problems (COPs) including their distributed versions (DCSPs and DCOPs)~\cite{schiex:94,Yokoo00,SilaghiY09}. Nogoods are value assignments to subsets of variables that cannot lead to a solution because they violate one or more constraints. However, the role of nogoods differ substantially. Nogoods are used to speed up the search of solutions by pruning portions of the search space that can be ignored. However, grounded constraints are used as solutions themselves.

There is also a great deal of literature on Explainable CSP~\cite{liffiton2008algorithms,laborie14,GuptaGO21}. A common theme across many of these methods is their use of hitting sets to find minimal unsatisfiable sets (MUSes) to explain why a CSP is unsatisfiable. An MUS is a set of constraints that is unsatisfiable but any of its subsets are satisfiable. Therefore, similar to our goal of finding shortest explanations, MUSes are subset-minimal justifications for why a CSP is unsatisfiable. Our approach can thus be viewed as a generalization of MUSes to weighted constraints, where constraints are no longer Boolean, but have associated costs instead. The approach to use MUSes as explanations have also been explored in other explainable subareas, such as explainable planning and scheduling~\cite{vas21,vasileioua2023lasp}, where MUSes correspond to sets of logical rules and constraints that must be satisfied by a plan or schedule. Finally, it is important to note that the existing explainable CSP, planning, and scheduling approaches discussed above are generally centralized approaches, while CEDAR is a distributed framework.


\section{Conclusions}

In this paper, we introduced the Explainable Distributed Constraint Optimization Problem (X-DCOP) model, an extension of the DCOPs that incorporates contrastive queries and explanations to enhance human interpretability. By addressing a critical gap in the adoption of DCOPs in the real world, X-DCOP facilitates understanding of distributed solutions, particularly in domains where human stakeholders are directly impacted. We defined key properties for valid explanations, presented theoretical guarantees for their existence, and proposed CEDAR, a distributed framework with optimized and suboptimal variants for computing explanations. Our empirical evaluations demonstrated the scalability of the framework and the trade-offs between explanation length and runtime. Furthermore, a user study confirmed the preference for concise explanations, underscoring the practical utility of the model. By bridging the fields of explainable AI and multi-agent systems, X-DCOP paves the way for more transparent, trustworthy, and user-aligned distributed optimization solutions, enabling broader application in real-world scenarios. 

 

\vskip 0.2in
\bibliography{X_Dcop}

\begin{thebibliography}{}

\bibitem[\protect\BCAY{Bowring, Pearce, Portway, Jain,\ \BBA\ Tambe}{Bowring et~al.}{2008}]{BowringPPJT08}
Bowring, E., Pearce, J.~P., Portway, C., Jain, M., \BBA\ Tambe, M. \BBOP2008\BBCP.
\newblock \BBOQ On \emph{k}-optimal distributed constraint optimization algorithms: new bounds and algorithms\BBCQ\
\newblock In {\Bem Proceedings of the International Conference on Autonomous Agents and Multiagent Systems (AAMAS)}, \BPGS\ 607--614.

\bibitem[\protect\BCAY{Byrne}{Byrne}{2023}]{byrne2023good}
Byrne, R. \BBOP2023\BBCP.
\newblock \BBOQ Good explanations in explainable artificial intelligence ({XAI}): Evidence from human explanatory reasoning\BBCQ\
\newblock In {\Bem Proceedings of the International Joint Conference on Artificial Intelligence (IJCAI)}, \BPGS\ 6536--6544.

\bibitem[\protect\BCAY{Fioretto, Pontelli,\ \BBA\ Yeoh}{Fioretto et~al.}{2018}]{fioretto:18}
Fioretto, F., Pontelli, E., \BBA\ Yeoh, W. \BBOP2018\BBCP.
\newblock \BBOQ Distributed constraint optimization problems and applications: A survey\BBCQ\
\newblock {\Bem Journal of Artificial Intelligence Research}, {\Bem 61}, 623--698.

\bibitem[\protect\BCAY{Fioretto, Yeoh,\ \BBA\ Pontelli}{Fioretto et~al.}{2017}]{fioretto:17a}
Fioretto, F., Yeoh, W., \BBA\ Pontelli, E. \BBOP2017\BBCP.
\newblock \BBOQ A multiagent system approach to scheduling devices in smart homes\BBCQ\
\newblock In {\Bem Proceedings of the International Conference on Autonomous Agents and Multiagent Systems (AAMAS)}, \BPGS\ 981--989.

\bibitem[\protect\BCAY{Gunning, Stefik, Choi, Miller, Stumpf,\ \BBA\ Yang}{Gunning et~al.}{2019}]{gunning2019xai}
Gunning, D., Stefik, M., Choi, J., Miller, T., Stumpf, S., \BBA\ Yang, G.-Z. \BBOP2019\BBCP.
\newblock \BBOQ {XAI}-explainable artificial intelligence\BBCQ\
\newblock {\Bem Science Robotics}, {\Bem 4\/}(37).

\bibitem[\protect\BCAY{Gupta, Genc,\ \BBA\ O'Sullivan}{Gupta et~al.}{2021}]{GuptaGO21}
Gupta, S.~D., Genc, B., \BBA\ O'Sullivan, B. \BBOP2021\BBCP.
\newblock \BBOQ Explanation in constraint satisfaction: {A} survey\BBCQ\
\newblock In {\Bem Proceedings of the Thirtieth International Joint Conference on Artificial Intelligence (IJCAI)}, \BPGS\ 4400--4407.

\bibitem[\protect\BCAY{Krarup, Krivic, Magazzeni, Long, Cashmore,\ \BBA\ Smith}{Krarup et~al.}{2021}]{krarup2021contrastive}
Krarup, B., Krivic, S., Magazzeni, D., Long, D., Cashmore, M., \BBA\ Smith, D. \BBOP2021\BBCP.
\newblock \BBOQ Contrastive explanations of plans through model restrictions\BBCQ\
\newblock {\Bem Journal of Artificial Intelligence Research}, {\Bem 72}, 533--612.

\bibitem[\protect\BCAY{Laborie}{Laborie}{2014}]{laborie14}
Laborie, P. \BBOP2014\BBCP.
\newblock \BBOQ An optimal iterative algorithm for extracting {MUC}s in a black-box constraint network\BBCQ\
\newblock In {\Bem Proceedings of the European Conference on Artificial Intelligence (ECAI)}, \BPGS\ 1051--1052.

\bibitem[\protect\BCAY{Liffiton\ \BBA\ Sakallah}{Liffiton\ \BBA\ Sakallah}{2008}]{liffiton2008algorithms}
Liffiton, M.\BBACOMMA\  \BBA\ Sakallah, K.~A. \BBOP2008\BBCP.
\newblock \BBOQ Algorithms for computing minimal unsatisfiable subsets of constraints\BBCQ\
\newblock {\Bem Journal of Automated Reasoning}, {\Bem 40}, 1--33.

\bibitem[\protect\BCAY{Maheswaran, Pearce,\ \BBA\ Tambe}{Maheswaran et~al.}{2004a}]{maheswaran:04b}
Maheswaran, R., Pearce, J., \BBA\ Tambe, M. \BBOP2004a\BBCP.
\newblock \BBOQ Distributed algorithms for {DCOP}: A graphical game-based approach\BBCQ\
\newblock In {\Bem Proceedings of the International Conference on Parallel and Distributed Computing Systems (PDCS)}, \BPGS\ 432--439.

\bibitem[\protect\BCAY{Maheswaran, Tambe, Bowring, Pearce,\ \BBA\ Varakantham}{Maheswaran et~al.}{2004b}]{maheswaran:04a}
Maheswaran, R., Tambe, M., Bowring, E., Pearce, J., \BBA\ Varakantham, P. \BBOP2004b\BBCP.
\newblock \BBOQ Taking {DCOP} to the real world: Efficient complete solutions for distributed event scheduling\BBCQ\
\newblock In {\Bem Proceedings of the International Conference on Autonomous Agents and Multiagent Systems (AAMAS)}, \BPGS\ 310--317.

\bibitem[\protect\BCAY{Miller, Ramchurn,\ \BBA\ Rogers}{Miller et~al.}{2012}]{miller:12}
Miller, S., Ramchurn, S., \BBA\ Rogers, A. \BBOP2012\BBCP.
\newblock \BBOQ Optimal decentralised dispatch of embedded generation in the smart grid\BBCQ\
\newblock In {\Bem Proceedings of the International Conference on Autonomous Agents and Multiagent Systems (AAMAS)}, \BPGS\ 281--288.

\bibitem[\protect\BCAY{Miller}{Miller}{2019}]{Miller2019}
Miller, T. \BBOP2019\BBCP.
\newblock \BBOQ Explanation in artificial intelligence: Insights from the social sciences\BBCQ\
\newblock {\Bem Artificial Intelligence}, {\Bem 267}, 1--38.

\bibitem[\protect\BCAY{Modi, Shen, Tambe,\ \BBA\ Yokoo}{Modi et~al.}{2005}]{modi:05}
Modi, P., Shen, W.-M., Tambe, M., \BBA\ Yokoo, M. \BBOP2005\BBCP.
\newblock \BBOQ {ADOPT}: Asynchronous distributed constraint optimization with quality guarantees\BBCQ\
\newblock {\Bem {A}rtificial {I}ntelligence}, {\Bem 161\/}(1--2), 149--180.

\bibitem[\protect\BCAY{Netzer, Grubshtein,\ \BBA\ Meisels}{Netzer et~al.}{2012}]{NetzerGM12}
Netzer, A., Grubshtein, A., \BBA\ Meisels, A. \BBOP2012\BBCP.
\newblock \BBOQ Concurrent forward bounding for distributed constraint optimization problems\BBCQ\
\newblock {\Bem Artificial Intelligence}, {\Bem 193}, 186--216.

\bibitem[\protect\BCAY{Palan\ \BBA\ Schitter}{Palan\ \BBA\ Schitter}{2018}]{palan2018prolific}
Palan, S.\BBACOMMA\  \BBA\ Schitter, C. \BBOP2018\BBCP.
\newblock \BBOQ Prolific.ac -- a subject pool for online experiments\BBCQ\
\newblock {\Bem Journal of Behavioral and Experimental Finance}, {\Bem 17}, 22--27.

\bibitem[\protect\BCAY{Pearce\ \BBA\ Tambe}{Pearce\ \BBA\ Tambe}{2007}]{pearce:07}
Pearce, J.~P.\BBACOMMA\  \BBA\ Tambe, M. \BBOP2007\BBCP.
\newblock \BBOQ Quality guarantees on k-optimal solutions for distributed constraint optimization problems\BBCQ\
\newblock In {\Bem Proceedings of the International Joint Conference on Artificial Intelligence (IJCAI)}, \BPGS\ 1446--1451.

\bibitem[\protect\BCAY{Pertzovskiy, Zivan,\ \BBA\ Agmon}{Pertzovskiy et~al.}{2023}]{pertzovskiy2023cams}
Pertzovskiy, A., Zivan, R., \BBA\ Agmon, N. \BBOP2023\BBCP.
\newblock \BBOQ {CAMS}: Collision avoiding max-sum for mobile sensor teams\BBCQ\
\newblock In {\Bem Proceedings of the International Conference on Autonomous Agents and Multiagent Systems (AAMAS)}, \BPGS\ 104--112.

\bibitem[\protect\BCAY{Petcu\ \BBA\ Faltings}{Petcu\ \BBA\ Faltings}{2005}]{petcu:05}
Petcu, A.\BBACOMMA\  \BBA\ Faltings, B. \BBOP2005\BBCP.
\newblock \BBOQ A scalable method for multiagent constraint optimization\BBCQ\
\newblock In {\Bem Proceedings of the International Joint Conference on Artificial Intelligence (IJCAI)}, \BPGS\ 1413--1420.

\bibitem[\protect\BCAY{Rust, Picard,\ \BBA\ Ramparany}{Rust et~al.}{2022}]{rust2022resilient}
Rust, P., Picard, G., \BBA\ Ramparany, F. \BBOP2022\BBCP.
\newblock \BBOQ Resilient distributed constraint reasoning to autonomously configure and adapt iot environments\BBCQ\
\newblock {\Bem ACM Transactions on Internet Technology}, {\Bem 22\/}(4), 1--31.

\bibitem[\protect\BCAY{Schiex\ \BBA\ Verfaillie}{Schiex\ \BBA\ Verfaillie}{1994}]{schiex:94}
Schiex, T.\BBACOMMA\  \BBA\ Verfaillie, G. \BBOP1994\BBCP.
\newblock \BBOQ Nogood recording for static and dynamic constraint satisfaction problems\BBCQ\
\newblock {\Bem International Journal of Artificial Intelligence Tools}, {\Bem 3\/}(2), 187--207.

\bibitem[\protect\BCAY{Silaghi\ \BBA\ Yokoo}{Silaghi\ \BBA\ Yokoo}{2009}]{SilaghiY09}
Silaghi, M.\BBACOMMA\  \BBA\ Yokoo, M. \BBOP2009\BBCP.
\newblock \BBOQ {ADOPT}-ing: unifying asynchronous distributed optimization with asynchronous backtracking\BBCQ\
\newblock {\Bem Autonomous Agents and Multi-Agent Systems}, {\Bem 19\/}(2), 89--123.

\bibitem[\protect\BCAY{Sreedharan, Chakraborti,\ \BBA\ Kambhampati}{Sreedharan et~al.}{2020}]{ttathai2020}
Sreedharan, S., Chakraborti, T., \BBA\ Kambhampati, S. \BBOP2020\BBCP.
\newblock \BBOQ The emerging landscape of explainable automated planning \& decision making\BBCQ\
\newblock In {\Bem Proceedings of the International Joint Conference on Artificial Intelligence (IJCAI)}, \BPGS\ 4803--4811.

\bibitem[\protect\BCAY{Vasileiou, Previti,\ \BBA\ Yeoh}{Vasileiou et~al.}{2021}]{vas21}
Vasileiou, S.~L., Previti, A., \BBA\ Yeoh, W. \BBOP2021\BBCP.
\newblock \BBOQ On exploiting hitting sets for model reconciliation\BBCQ\
\newblock In {\Bem Proceedings of the AAAI Conference on Artificial Intelligence (AAAI)}, \BPGS\ 6514--6521.

\bibitem[\protect\BCAY{Vasileiou, Xu,\ \BBA\ Yeoh}{Vasileiou et~al.}{2023}]{vasileioua2023lasp}
Vasileiou, S.~L., Xu, B., \BBA\ Yeoh, W. \BBOP2023\BBCP.
\newblock \BBOQ A logic-based framework for explainable agent scheduling problems\BBCQ\
\newblock In {\Bem Proceedings of the European Conference on Artificial Intelligence (ECAI)}, \BPGS\ 2402--2410.

\bibitem[\protect\BCAY{Yeoh, Felner,\ \BBA\ Koenig}{Yeoh et~al.}{2010}]{yeoh:10}
Yeoh, W., Felner, A., \BBA\ Koenig, S. \BBOP2010\BBCP.
\newblock \BBOQ {BnB-ADOPT}: An asynchronous branch-and-bound {DCOP} algorithm\BBCQ\
\newblock {\Bem Journal of Artificial Intelligence Research}, {\Bem 38}, 85--133.

\bibitem[\protect\BCAY{Yokoo\ \BBA\ Hirayama}{Yokoo\ \BBA\ Hirayama}{2000}]{Yokoo00}
Yokoo, M.\BBACOMMA\  \BBA\ Hirayama, K. \BBOP2000\BBCP.
\newblock {\Bem Distributed Constraint Satisfaction Problems}.
\newblock Springer Verlag.

\bibitem[\protect\BCAY{Zehtabi, Pozanco, Bolch, Borrajo,\ \BBA\ Kraus}{Zehtabi et~al.}{2024}]{zehtabi2024contrastive}
Zehtabi, P., Pozanco, A., Bolch, A., Borrajo, D., \BBA\ Kraus, S. \BBOP2024\BBCP.
\newblock \BBOQ Contrastive explanations of centralized multi-agent optimization solutions\BBCQ\
\newblock In {\Bem Proceedings of the International Conference on Automated Planning and Scheduling (ICAPS)}, \BPGS\ 671--679.

\bibitem[\protect\BCAY{Zilberstein, Rao, Salis,\ \BBA\ Chien}{Zilberstein et~al.}{2024}]{ZilbersteinRSC24}
Zilberstein, I., Rao, A., Salis, M., \BBA\ Chien, S.~A. \BBOP2024\BBCP.
\newblock \BBOQ Decentralized, decomposition-based observation scheduling for a large-scale satellite constellation\BBCQ\
\newblock In {\Bem Proceedings of the International Conference on Automated Planning and Scheduling (ICAPS)}, \BPGS\ 716--724.

\bibitem[\protect\BCAY{Zivan\ \BBA\ Meisels}{Zivan\ \BBA\ Meisels}{2006}]{ZivanM06b}
Zivan, R.\BBACOMMA\  \BBA\ Meisels, A. \BBOP2006\BBCP.
\newblock \BBOQ Message delay and {DisCSP} search algorithms\BBCQ\
\newblock {\Bem Annals of Mathematics and Artificial Intelligence}, {\Bem 46}, 415--439.

\bibitem[\protect\BCAY{Zivan, Yedidsion, Okamoto, Glinton,\ \BBA\ Sycara}{Zivan et~al.}{2015}]{zivan:15}
Zivan, R., Yedidsion, H., Okamoto, S., Glinton, R., \BBA\ Sycara, K. \BBOP2015\BBCP.
\newblock \BBOQ Distributed constraint optimization for teams of mobile sensing agents\BBCQ\
\newblock {\Bem Journal of Autonomous Agents and Multi-Agent Systems}, {\Bem 29\/}(3), 495--536.

\end{thebibliography}
\bibliographystyle{theapa}

\section*{Appendix A. Full Set of Computational Results}

In this appendix, we present the full set of results on meeting scheduling and random uniform problems (both sparse and dense). Throughout the figures, we show both random and best alternative queries, represented by dashed and solid curves, respectively. This section is organized by metric and the scale variable examined for each type of problem. The first set of figures illustrates how 1-Opt and optimal solutions scale as a function of problem size, and variables in query size. In the second part, we present the differences in performance between CEDAR, its optimizations, and its variants.

Figures \ref{fig:1_results_percent_query_meeting}, \ref{fig:1_results_percent_query_sparse}, and \ref{fig:1_results_percent_query_dense} demonstrate that, for best alternative queries, all instance explanations are valid for 1-Opt solutions when $\var{\sigma_Q} = 1$. However, as $\var{\sigma_Q}$ increases, the number of valid instances decreases at different rates depending on the type of problem. For optimal solutions, valid explanations are obtained regardless of the size of $\var{\sigma_Q}$. These results empirically confirm Theorem $1$. For random alternative values, valid explanations were found for both 1-opt and optimal solution for sizes of $\var{\sigma_Q}$. Although the 1-Opt solution does not guarantee an optimal outcome, it is highly likely to yield an improved solution compared to one selected randomly.

Figures \ref{fig:5_results_percent_A_meeting}, \ref{fig:5_results_percent_A_sparse}, and \ref{fig:5_results_percent_A_dense} demonstrate the scalability of CEDAR, as it is able to handle large problems with up to $50$ agents. Results for X-DCOPs with optimal DCOP solutions are unavailable for $15$ or more agents, as such optimal solutions were not found due to the scalability issues of existing complete DCOP solvers. For 1-Opt solutions, the percentage of valid explanations increases as the number of agents increases for the best alternative queries. It is empirically evident that, for $|A| > 35$, a valid explanation is provided for all instances, even for queries with $|\var{\sigma_Q}| > 1$, across all DCOP problem types.

In Figures \ref{fig:7_results_num_constraints_A_meeting}, \ref{fig:7_results_num_constraints_A_sparse}, and \ref{fig:7_results_num_constraints_A_dense}, we show the number of constraints in explanations as a function of the number of agents. It shows that explanations for larger problems require more grounded constraints. Random alternative queries generally require fewer constraints, especially in meeting scheduling problems, where infinite or highly costly constraints often make a single constraint sufficient for a valid explanation.

Figures \ref{fig:6_results_NCLO_A_meeting}, \ref{fig:6_results_NCLO_A_sparse}, and \ref{fig:6_results_NCLO_A_dense} illustrate how NCLOs for CEDAR(O1) scale with the number of agents. Regardless of query type or graph topology, NCLOs increase linearly with the number of agents. This aligns with Figures \ref{fig:7_results_num_constraints_A_meeting}, \ref{fig:7_results_num_constraints_A_sparse}, and \ref{fig:7_results_num_constraints_A_dense}, as more agents in the problems lead to more grounded constraints in explanation sets, increasing computational demands.

Next, we illustrate how CEDAR, its optimizations, and variants differ in their computational requirements. Figures \ref{fig:3_results_num_constraints_Q_meeting}, \ref{fig:3_results_num_constraints_Q_sparse}, and \ref{fig:3_results_num_constraints_Q_dense} show the explanation length (i.e., number of constraints in explanation) as a function of query size. As the number of query variables increases, valid explanations generally grow longer, especially for best alternative queries, which tend to present closer-cost alternatives to the solution, resulting in more grounded constraints in the explanations. Compared to random alternative queries, best alternative queries typically present alternatives with costs closer to the solution, resulting in smaller constraint costs. Consequently, the number of constraints required for these explanations is larger than for random alternative queries.

Next, we illustrate how CEDAR, its optimizations, and variants differ in their computational requirements. Figures \ref{fig:3_results_num_constraints_Q_meeting}, \ref{fig:3_results_num_constraints_Q_sparse}, and \ref{fig:3_results_num_constraints_Q_dense} demonstrate the relationship between the number of constraints in explanations (optimality) and query size. Across all problem types, the length of valid explanations tends to increase with the complexity of the queries, as measured by the number of query variables ($|\var{\sigma_Q}|$). This trend is more pronounced for best alternative queries, which often present alternatives with costs closer to the solution. These closer-cost alternatives necessitate more grounded constraints in the explanations to achieve validity. As the number of variables in the query increases, the alternatives tend to align more closely with the solution's cost. Consequently, the number of constraints required for these explanations increases. Across all types of problems Within each query type, CEDAR finds the longest explanations since it includes the full set of grounded constraints $\setf{F}_{\downarrow \hat{\sigma}_Q}$ for the alternative solution $\hat{\sigma}_Q$ in its explanations. CEDAR with optimizations O1 and O2 find shortest explanations, as expected since they optimize for explanation length. CEDAR with suboptimal variants V1 and V2 find longer explanations, but they are generally shorter than CEDAR's.

Figures~\ref{fig:2_NCLO_Q_meeting}, \ref{fig:2_NCLO_Q_sparse}, and \ref{fig:2_NCLO_Q_dense} shows that the runtimes increase as the queries increases in complexity, where runtimes for best alternative queries are larger than runtimes for random queries across all types of problems. This trend is more pronounced for variant V2, where grounded explanations are potentially collected by only a subset of agents rather than all of them. This is particularly evident for random alternatives, where higher constraint costs make valid explanations easier to identify. Apparently, fewer agents need to be included in the explanation process, leading to reduced runtimes for these queries. Within each query type, CEDAR is the fastest. CEDAR with optimizations O1 and O2 are the slowest as they sort grounded constraints to identify shortest explanations. When $|\var{\sigma_Q}|$ is small, optimization O2 is faster because $\log(|\mathcal{F}_a|) < |\mathcal{A}_Q| \cdot \log(|\mathcal{A}_Q|)$ (see description of O2 in Section 4.1). However, as $|\var{\sigma_Q}|$ increases, $\log(|\mathcal{F}_a|)$ increases, and O1 loses its speedup compared to O2. Finally, CEDAR with variants V1 and V2 have smaller runtimes, but they are still larger than CEDAR's. Therefore, they show the empirical tradeoffs made by CEDAR with V1 and V2 between solution quality in terms of explanation length and runtime in terms of NCLOs.

Figures~\ref{fig:4_NCLO_Q_meeting}, \ref{fig:4_NCLO_sparse}, and \ref{fig:4_NCLO_Q_dense} show the difference in cost between the grounded constraints in the solution set and those in the alternative set as a function of NCLOs for $|\var{\sigma_Q}| = 10$. For meeting scheduling problems, the sigmoid value of the cost difference is presented, as spatial constraints are represented by large costs. To ensure consistency when comparing the average cost differences across all instances of X-DCOP, it is assumed that $a_Q$ knows the costs of the solution set constraints at $NCLO = 0$. The increase in cost difference over runtime indicates the point at which alternative constraints begin to be incrementally added to the set. Since the solutions are optimal, all types of CEDAR yield a non-negative cost difference at termination (alternative constraints have costs that are higher than or equal to the costs of the solution constraints, as described in Definition 2). Consistent with Figures~\ref{fig:2_NCLO_Q_meeting}, \ref{fig:2_NCLO_Q_sparse}, and \ref{fig:2_NCLO_Q_dense}, CEDAR O1 and O2 require the longest runtimes for best alternative queries. A tradeoff emerges between CEDAR O1's preprocessing step, sorting all alternative constraint sets, and CEDAR O2's dynamic adjustment of the cost delta. CEDAR O1 maintains a consistent cost delta for a longer duration, as the preprocessing step minimizes the impact of adding alternative constraints on runtime. In contrast, CEDAR O2 exhibits a more gradual increase in cost delta, as $a_Q$ dynamically inserts and extracts constraints from the heap. This incremental approach increases computational overhead but results in a smoother cost difference growth. However, when examining random alternatives, CEDAR O2 demonstrates shorter runtimes. Such queries do not reflect the worst-case scenario, as the explanations require fewer alternative constraints. In these cases, CEDAR O2 requires fewer insertions of new constraints into the heap, enabling it to terminate faster.

\begin{figure}[h]
\centering
\small
\includegraphics[height=15em]{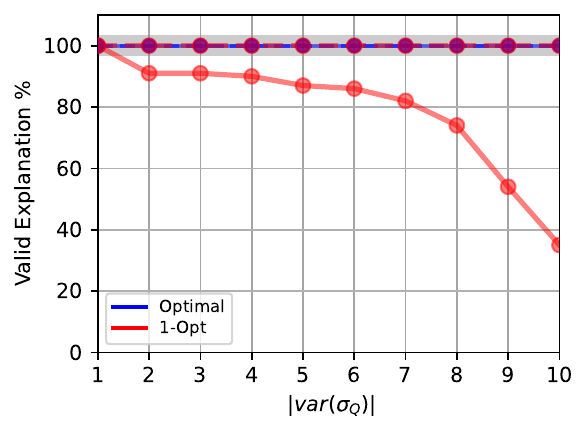} \\ 
\vspace{-0.5em} 
\caption{Percentage of instances for which valid explanations were found using CEDAR (O1) as a function of query variables $|\var{\sigma_Q}|$  with $|\mathcal{A}| = 10$ on 1-opt and optimal solutions for \textbf{meeting scheduling problems}.}
\label{fig:1_results_percent_query_meeting}
\end{figure}

\begin{figure}[h]
\centering
\small
\includegraphics[height=15em]{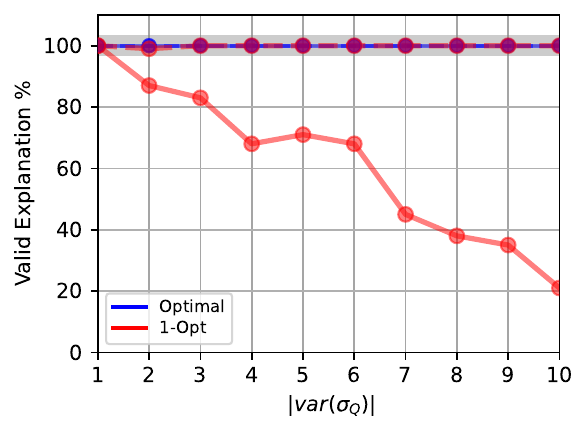} \\ 
\vspace{-0.5em} 
\caption{Percentage of instances for which valid explanations were found using CEDAR (O1) as a function of query variables $|\var{\sigma_Q}|$  with $|\mathcal{A}| = 10$ on 1-opt and optimal solutions for \textbf{sparse random uniform problems}.}
\label{fig:1_results_percent_query_sparse}
\end{figure}

\begin{figure}[h]
\centering
\small
\includegraphics[height=15em]{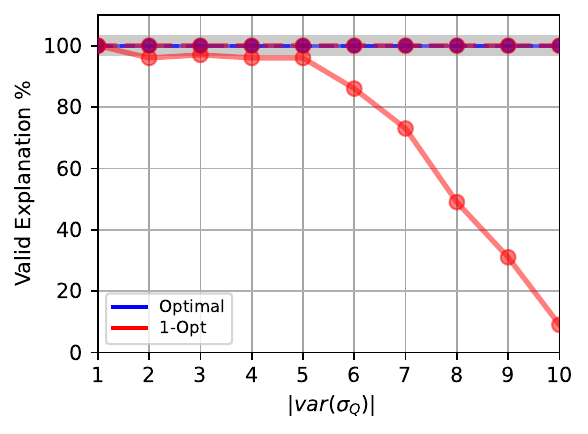} \\ 
\vspace{-0.5em} 
\caption{Percentage of instances for which valid explanations were found using CEDAR (O1) as a function of query variables $|\var{\sigma_Q}|$ 
with $|\mathcal{A}| = 10$ on 1-opt and optimal solutions for \textbf{dense random uniform problems}.}
\label{fig:1_results_percent_query_dense}
\end{figure}


\begin{figure}[h]
\centering
\small
\includegraphics[height=15em]{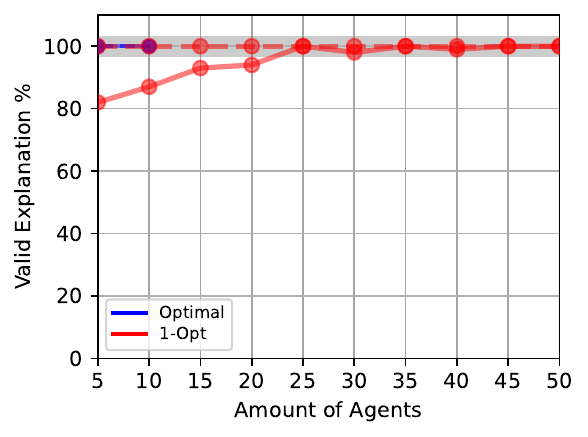} \\ 
\vspace{-0.5em} 
\caption{Percentage of instances for which valid explanations were found using CEDAR (O1) as a function of agents $|\mathcal{A}|$ with $|\var{\sigma_Q}| = 5$  on 1-opt and optimal solutions for \textbf{meeting scheduling problems}.}
\label{fig:5_results_percent_A_meeting}
\end{figure}

\begin{figure}[h]
\centering
\small
\includegraphics[height=15em]{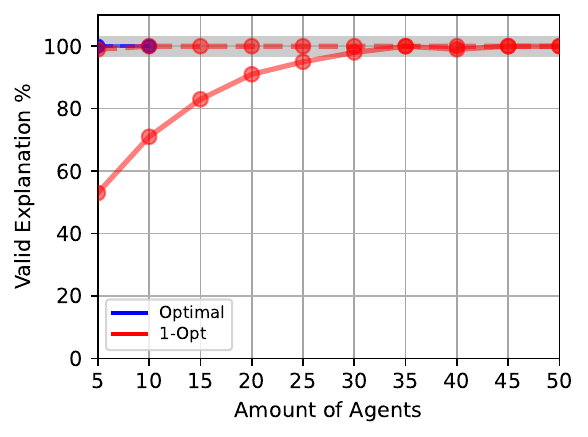} \\ 
\vspace{-0.5em} 
\caption{Percentage of instances for which valid explanations were found using CEDAR (O1) as a function of agents $|\mathcal{A}|$ with $|\var{\sigma_Q}| = 5$  on 1-opt and optimal solutions for \textbf{sparse random uniform problems}.}
\label{fig:5_results_percent_A_sparse}
\end{figure}

\begin{figure}[h]
\centering
\small
\includegraphics[height=15em]{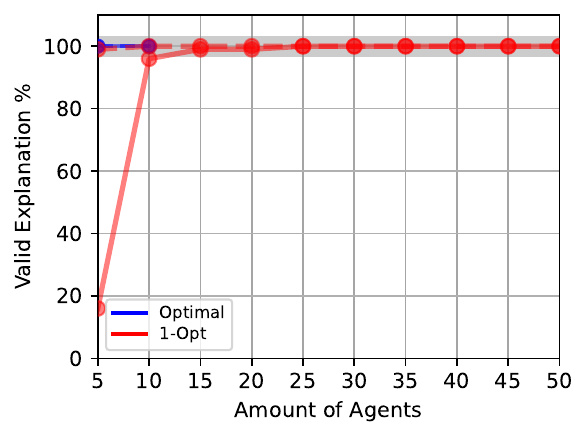} \\ 
\vspace{-0.5em} 
\caption{Percentage of instances for which valid explanations were found using CEDAR (O1) as a function of agents $|\mathcal{A}|$ with $|\var{\sigma_Q}| = 5$  on 1-opt and optimal solutions for \textbf{dense random uniform problems}.}
\label{fig:5_results_percent_A_dense}
\end{figure}


\begin{figure}[h]
\centering
\small
\includegraphics[height=15em]{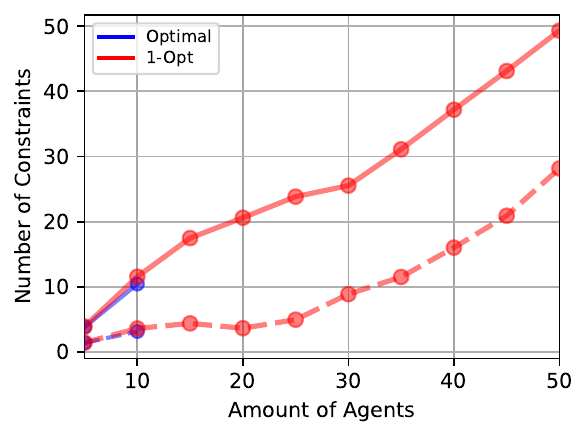} \\ 
\vspace{-0.5em} 
\caption{Number of grounded constraints in the explanations (using CEDAR(O1)) as a function of agents $|\mathcal{A}|$ with $|\var{\sigma_Q}| = 5$ for explanations generated using CEDAR(O1) on 1-opt and optimal solutions for \textbf{meeting scheduling problems}.}
\label{fig:7_results_num_constraints_A_meeting}
\end{figure}

\begin{figure}[h]
\centering
\small
\includegraphics[height=15em]{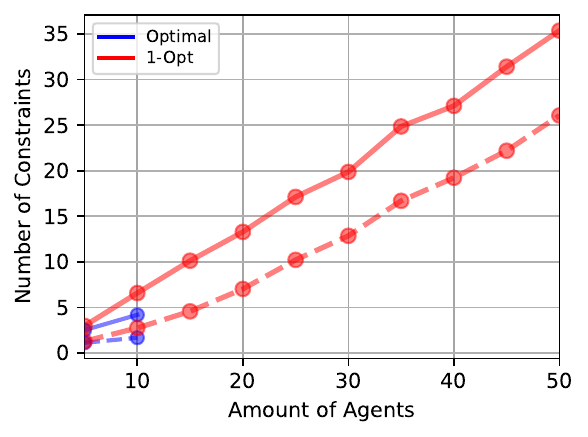} \\ 
\vspace{-0.5em} 
\caption{Number of grounded constraints in the explanations (using CEDAR(O1)) as a function of agents $|\mathcal{A}|$ with $|\var{\sigma_Q}| = 5$ for explanations generated using CEDAR(O1) on 1-opt and optimal solutions for \textbf{sparse random uniform problems}.}
\label{fig:7_results_num_constraints_A_sparse}
\end{figure}

\begin{figure}[h]
\centering
\small
\includegraphics[height=15em]{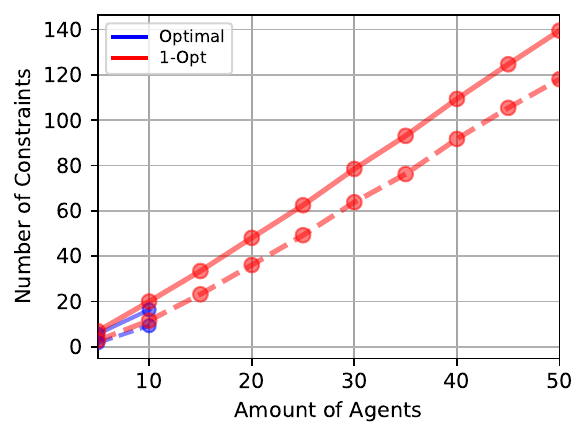} \\ 
\vspace{-0.5em} 
\caption{Number of grounded constraints in the explanations (using CEDAR(O1)) as a function of agents $|\mathcal{A}|$ with $|\var{\sigma_Q}| = 5$ for explanations generated using CEDAR(O1) on 1-opt and optimal solutions for \textbf{dense random uniform problems}.}
\label{fig:7_results_num_constraints_A_dense}
\end{figure}


\begin{figure}[h]
\centering
\small
\includegraphics[height=15em]{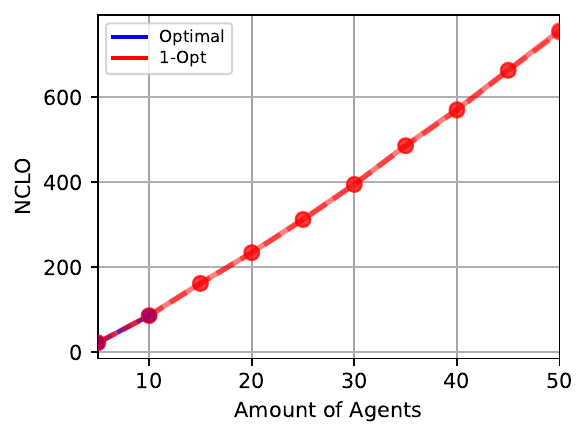} \\ 
\vspace{-0.5em} 
\caption{NCLO as a function of agents $|\mathcal{A}|$ with $|\var{\sigma_Q}| = 5$ for explanations generated using CEDAR(O1) on 1-opt and optimal solutions for \textbf{meeting scheduling problems}.}
\label{fig:6_results_NCLO_A_meeting}
\end{figure}

\begin{figure}[h]
\centering
\small
\includegraphics[height=15em]{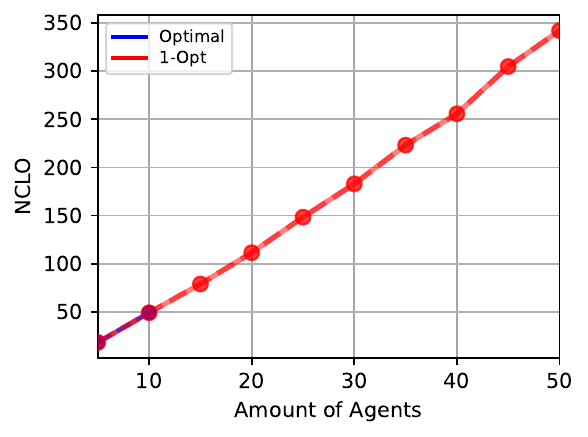} \\ 
\vspace{-0.5em} 
\caption{NCLO as a function of agents $|\mathcal{A}|$ with $|\var{\sigma_Q}| = 5$ for explanations generated using CEDAR(O1) on 1-opt and optimal solutions for \textbf{sparse random uniform problems}.}
\label{fig:6_results_NCLO_A_sparse}
\end{figure}

\begin{figure}[h]
\centering
\small
\includegraphics[height=15em]{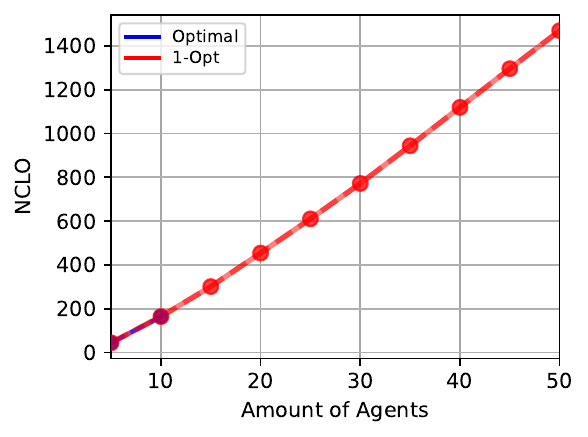} \\ 
\vspace{-0.5em} 
\caption{NCLO as a function of agents $|\mathcal{A}|$ with $|\var{\sigma_Q}| = 5$ for explanations generated using CEDAR(O1) on 1-opt and optimal solutions for \textbf{dense random uniform problems}.}
\label{fig:6_results_NCLO_A_dense}
\end{figure}


\begin{figure}[h]
\centering
\small
\includegraphics[height=15em]{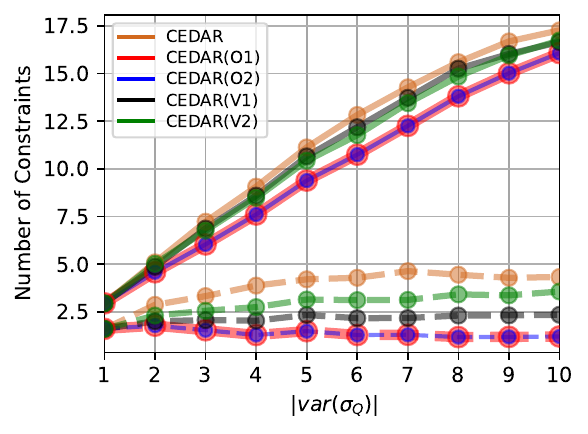} \\ 
\vspace{-0.5em} 
\caption{Number of grounded constraints in the explanations as a function of number of query variables $|\var{\sigma_Q}|$, for CEDAR, its optimized versions, and its variants on \textbf{meeting scheduling problems} with $10$ meetings for optimal complete solutions.}
\label{fig:3_results_num_constraints_Q_meeting}
\end{figure}

\begin{figure}[h]
\centering
\small
\includegraphics[height=15em]{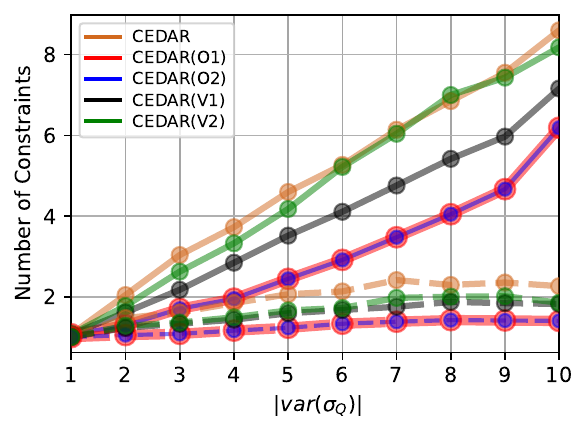} \\ 
\vspace{-0.5em} 
\caption{Number of grounded constraints in the explanations as a function of number of query variables $|\var{\sigma_Q}|$, for CEDAR, its optimized versions, and its variants on \textbf{sparse random uniform problems} with $10$ agents for optimal complete solutions.}

\label{fig:3_results_num_constraints_Q_sparse}
\end{figure}

\begin{figure}[h]
\centering
\small
\includegraphics[height=15em]{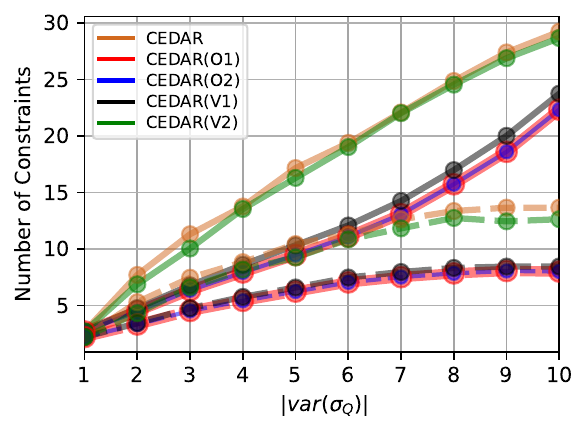} \\ 
\vspace{-0.5em} 
\caption{Number of grounded constraints in the explanations as a function of number of query variables $|\var{\sigma_Q}|$, for CEDAR, its optimized versions, and its variants on \textbf{dense random uniform problems} with $10$ agents for optimal complete solutions.}

\label{fig:3_results_num_constraints_Q_dense}
\end{figure}


\begin{figure}[h]
\centering
\small
\includegraphics[height=15em]{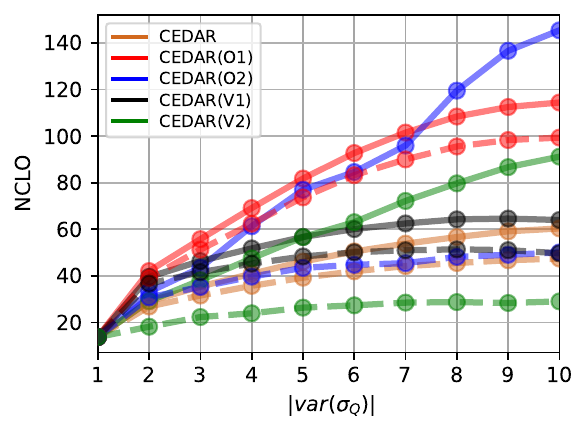} \\ 
\vspace{-0.5em} 
\caption{NCLO as a function of number of query variables $|\var{\sigma_Q}|$, for CEDAR, its optimized versions, and its variants on \textbf{meeting scheduling problems} with $10$ meetings for optimal complete solutions.}
\label{fig:2_NCLO_Q_meeting}
\end{figure}

\begin{figure}[h]
\centering
\small
\includegraphics[height=15em]{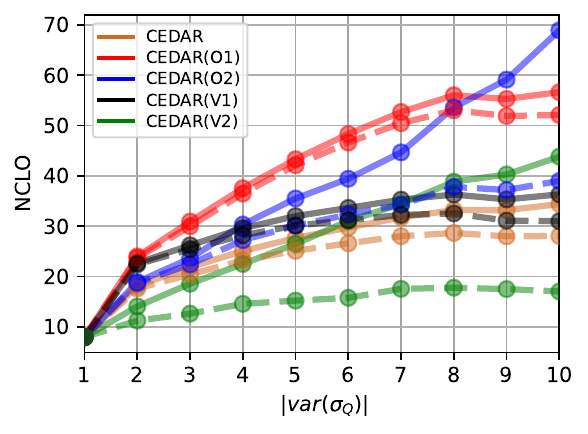} \\ 
\vspace{-0.5em} 
\caption{NCLO as a function of number of query variables $|\var{\sigma_Q}|$, for CEDAR, its optimized versions, and its variants on \textbf{sparse random uniform problems} with $10$ agents for optimal complete solutions.}

\label{fig:2_NCLO_Q_sparse}
\end{figure}

\begin{figure}[h]
\centering
\small
\includegraphics[height=15em]{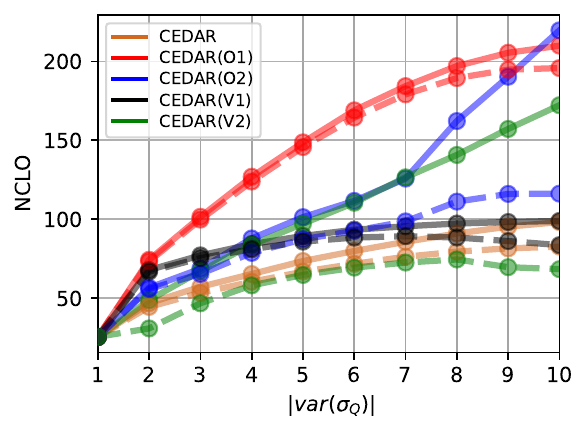} \\ 
\vspace{-0.5em} 
\caption{NCLO as a function of number of query variables $|\var{\sigma_Q}|$, for CEDAR, its optimized versions, and its variants on \textbf{dense random uniform problems} with $10$ agents for optimal complete solutions.}

\label{fig:2_NCLO_Q_dense}
\end{figure}


\begin{figure}[h]
\centering
\small
\includegraphics[height=15em]{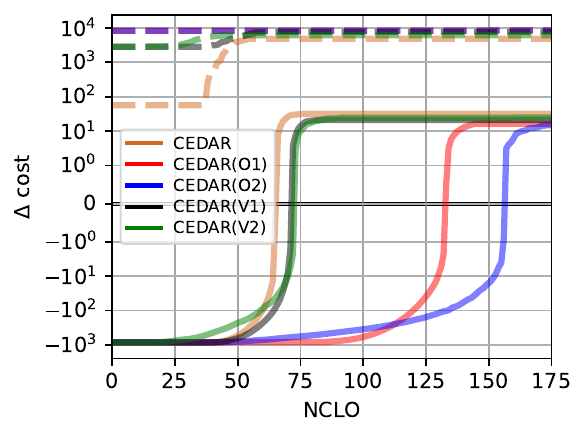} \\ 
\vspace{-0.5em} 
\caption{The sigmoid difference in costs $\setf{F}_{\downarrow \hat{\sigma}_Q}\!(\hat{\sigma}_Q) - \setf{F}_{\downarrow \sigma_Q}\!(\sigma_Q)$ as a function of NCLOs, for CEDAR, its optimized versions, and its variants on \textbf{meeting scheduling problems} with $10$ meetings for optimal complete solutions.}
\label{fig:4_NCLO_Q_meeting}
\end{figure}

\begin{figure}[h]
\centering
\small
\includegraphics[height=15em]{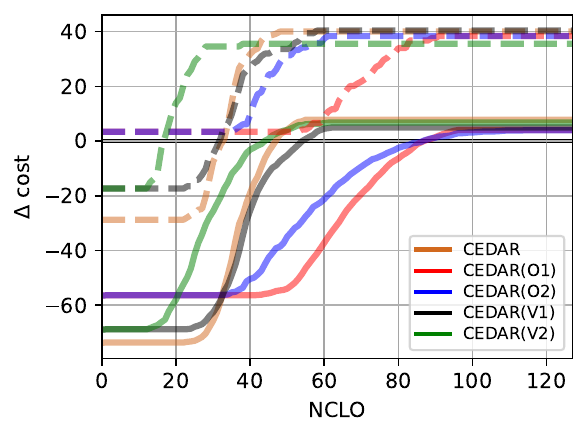} \\ 
\vspace{-0.5em} 
\caption{The difference in costs $\setf{F}_{\downarrow \hat{\sigma}_Q}\!(\hat{\sigma}_Q) - \setf{F}_{\downarrow \sigma_Q}\!(\sigma_Q)$ as a function of NCLOs, for CEDAR, its optimized versions, and its variants on \textbf{sparse random uniform problems} with $10$ agents for optimal complete solutions.}

\label{fig:4_NCLO_sparse}
\end{figure}

\begin{figure}[h]
\centering
\small
\includegraphics[height=15em]{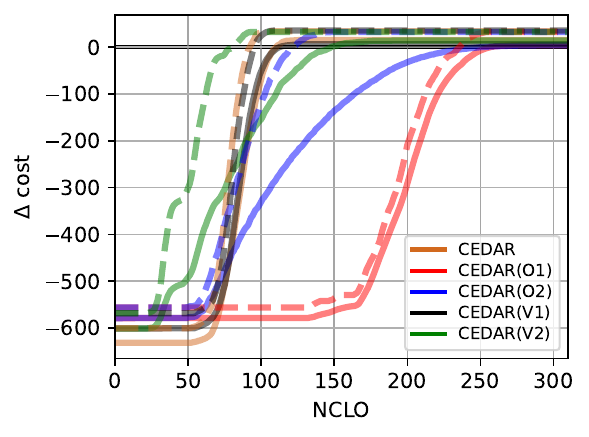} \\ 
\vspace{-0.5em} 
\caption{The difference in costs $\setf{F}_{\downarrow \hat{\sigma}_Q}\!(\hat{\sigma}_Q) - \setf{F}_{\downarrow \sigma_Q}\!(\sigma_Q)$ as a function of NCLOs, for CEDAR, its optimized versions, and its variants on \textbf{sparse dense uniform problems} with $10$ agents for optimal complete solutions.}
\label{fig:4_NCLO_Q_dense}
\end{figure}

\clearpage
\clearpage

\section*{Appendix B.Human User Study}

\xhdr{Study Design:} We created a synthetic meeting scheduling scenario designed to capture the essential elements of a DCOP-based scheduling problem while remaining accessible to participants without technical background. The scenario involved four participants (the user and three others - Bob, Charlie, and David) coordinating two meetings (M1 and M2) across four possible time slots: morning, noon, afternoon, and evening. To reflect realistic scheduling constraints, we structured the scenario such that some participants needed to attend both meetings while others had more limited participation requirements: the user and Bob were required to attend both meetings, while Charlie needed to attend only M1 and David only M2.

For this scenario, the scheduling system proposed holding M1 in the afternoon and M2 in the evening. Participants were presented with a contrastive query comparing this schedule to an alternative where M1 would be at noon and M2 in the afternoon. To systematically evaluate how explanation length affects user comprehension and preferences, we created two groups that were shown identical scenarios but received different pairs of explanations to compare:

\squishlist
    \item \textit{Group 1} compared explanations of length two versus three.
    \item \textit{Group 2} compared explanations of length two versus four.
\squishend

Here, the length of an explanation refers to the number of agents' preferences used to justify the scheduling decision.

For each pair of explanations, participants were asked two key questions: (Q1) which explanation was easier to understand; and (Q2) which explanation they would be more likely to provide when explaining the scheduling decision to others. 

The full user study shown to all participants is depicted progressively in Figures~\ref{intro} to~\ref{questions}.

\xhdr{Participant Recruitment and Demographics:} We recruited 130 participants through Prolific, randomly assigning 65 participants to each group.\footnote{The study was approved by our institution’s ethics board and adhered to the guidelines for responsible research practices.} All participants were proficient in English, and were compensated with a base payment of \$2.50 for completing the study. To ensure data quality, we included two attention check questions verifying participants' understanding of the scenario. After removing participants who failed both attention checks (1 from Group 1 and 2 from Group 2), our final analysis included 127 participants. The participants were diverse in terms of age (mean = 36.2 years), gender (56.7\% female, 42.5\% male, 0.8\% other), and education level (41.7\% undergraduate, 29.9\% graduate, 20.5\% high school, 7.9\% other). The median completion time was 6.01 minutes.

\begin{figure}
    \centering
    \includegraphics[width=0.7\columnwidth]{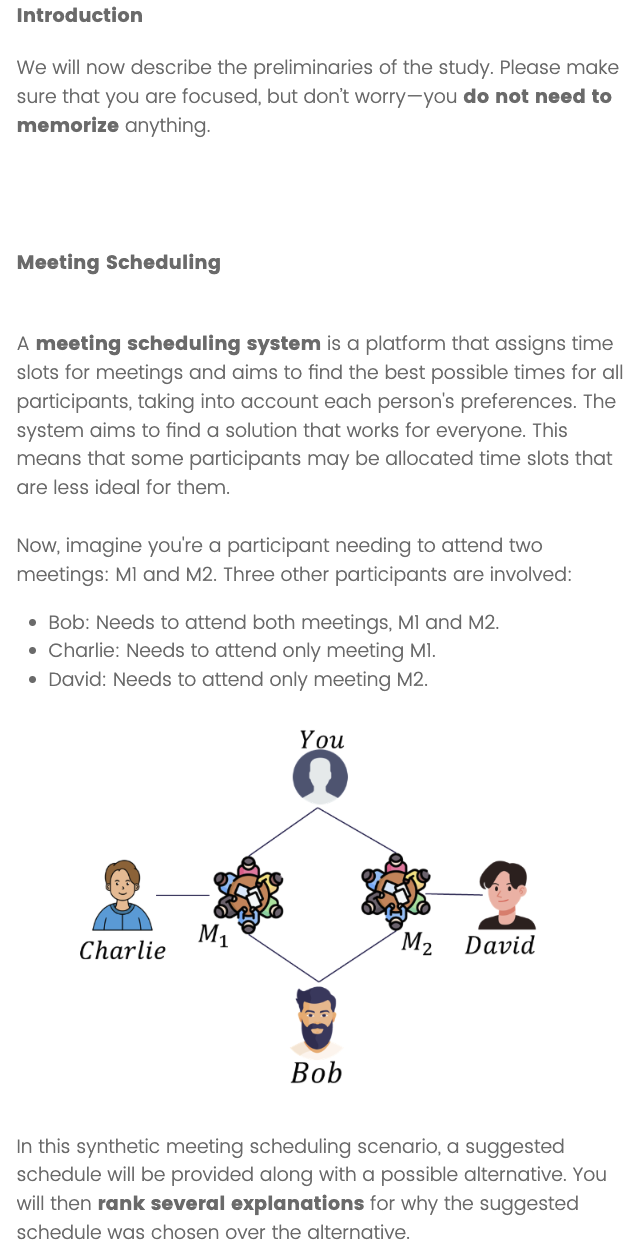}
    \caption{Introductory information about the scheduling problem.}
    \label{intro}
\end{figure}

\begin{figure}
    \centering
    \includegraphics[width=0.7\columnwidth]{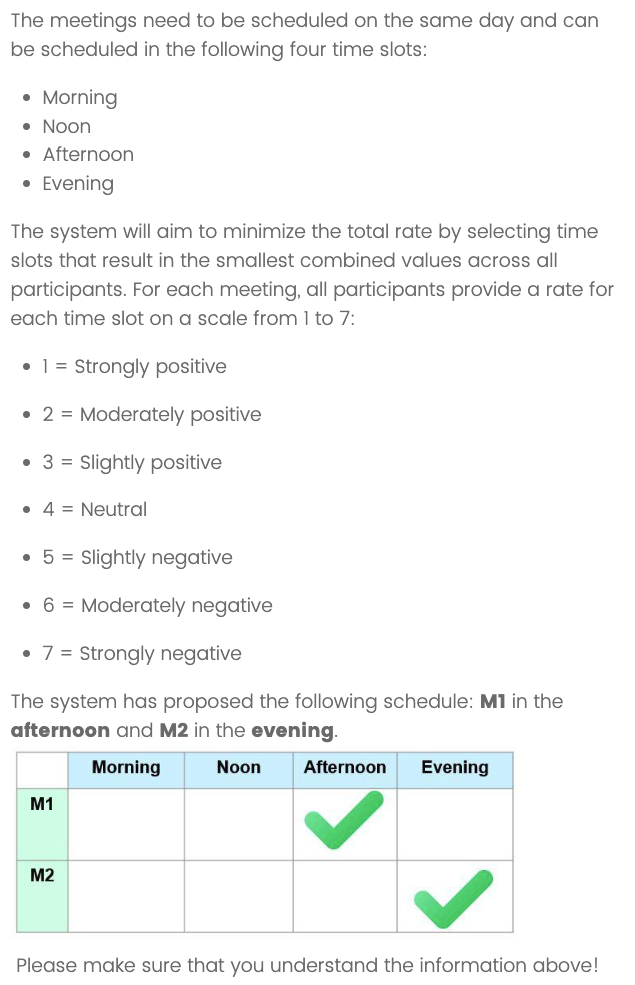}
    \caption{The proposed schedule by the system.}
    \label{schedule}
\end{figure}

\begin{figure}
    \centering
    \includegraphics[width=0.7\columnwidth]{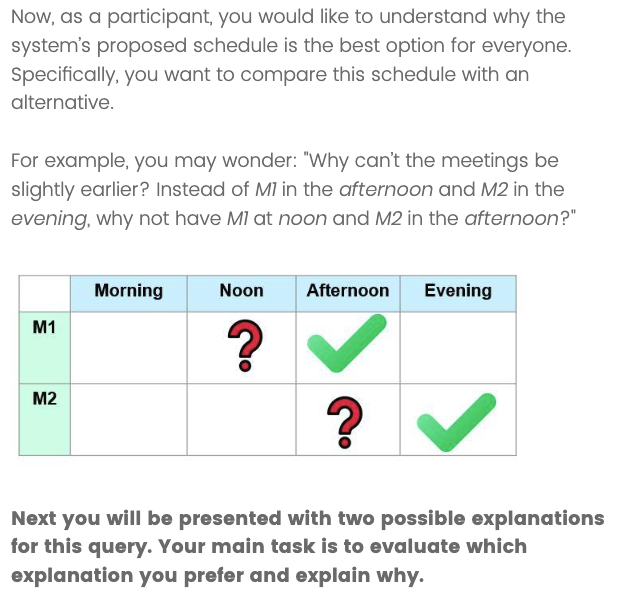}
    \caption{The query for the proposed schedule and description of the participants' upcoming task.}
    \label{query}
\end{figure}

\begin{figure}
    \centering
    \includegraphics[width=0.7\columnwidth]{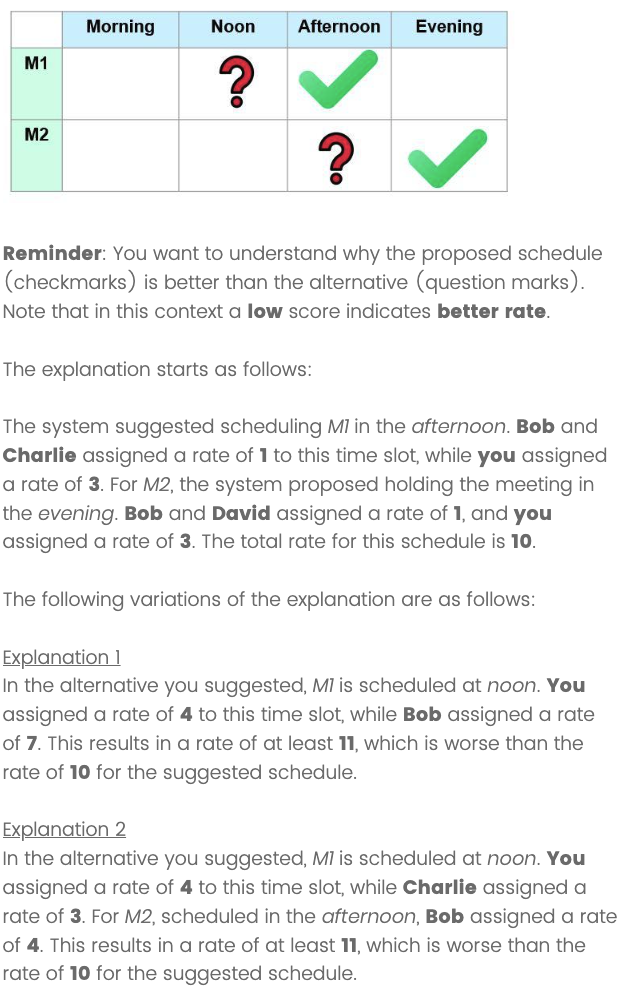}
    \caption{The explanations shown to the participants in Group 1. Note that Explanation 1 is of length two (i.e., two agent variables) and Explanation 2 is of length three (i.e., three agent variables).}
    \label{2x3}
\end{figure}

\begin{figure}
    \centering
    \includegraphics[width=0.7\columnwidth]{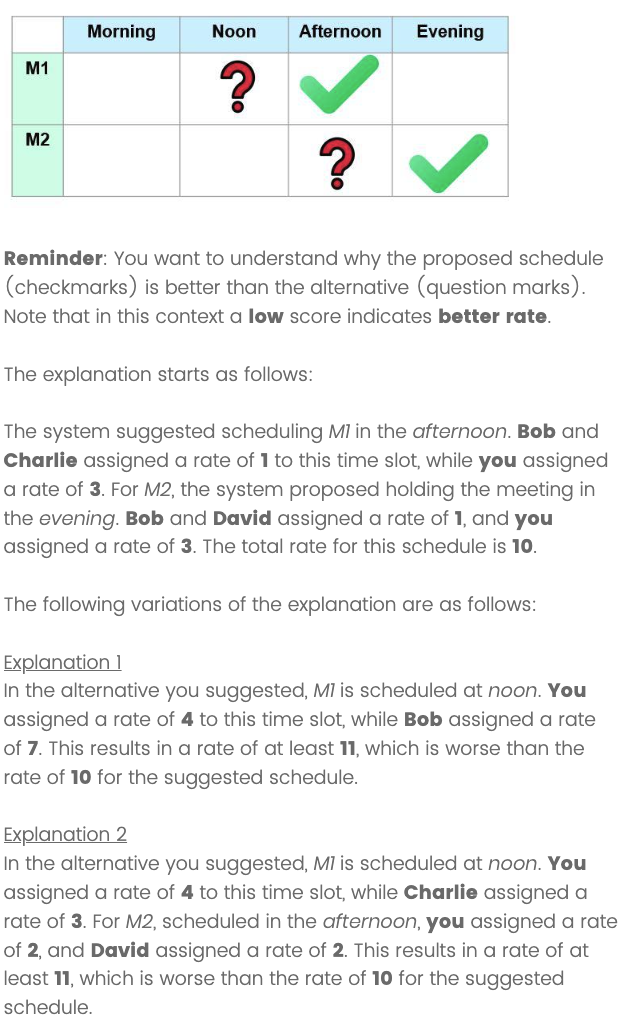}
    \caption{The explanations shown to the participants in Group 2. Note that Explanation 1 is of length two (i.e., two agent variables) and Explanation 2 is of length four (i.e., four agent variables).}
    \label{2x4}
\end{figure}

\begin{figure}
    \centering
    \includegraphics[width=0.8\linewidth]{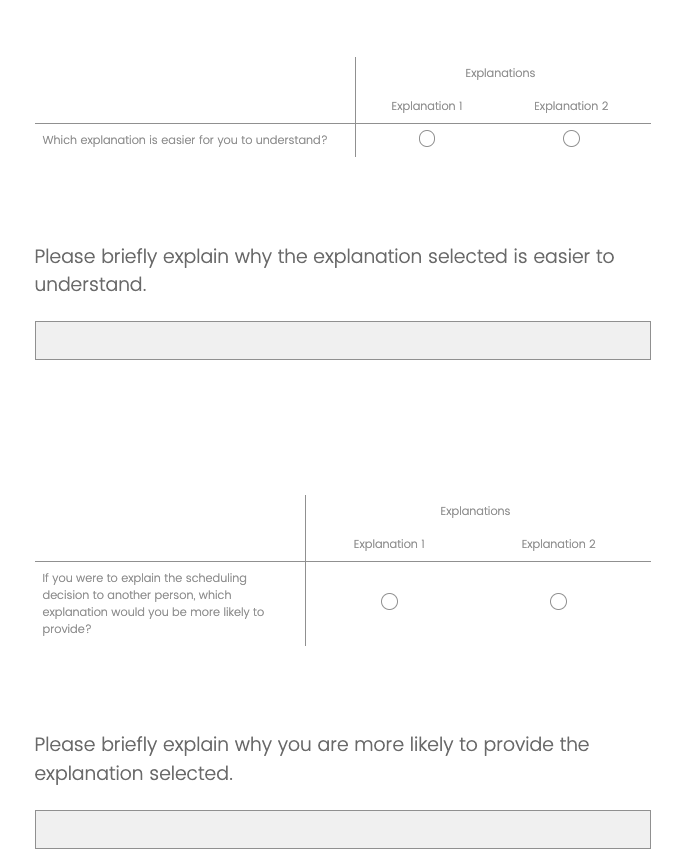}
    \caption{The evaluation questions shown to the participants in both groups.}
    \label{questions}
\end{figure}


\end{document}